\documentclass[letterpaper, 10 pt, conference]{ieeeconf}  % Comment this line out if you need a4paper
\pdfoutput=1

\IEEEoverridecommandlockouts                              % This command is only needed if 
\usepackage{times} % assumes new font selection scheme installed

\usepackage[ruled,linesnumbered]{algorithm2e}
\usepackage{algpseudocode}

\usepackage{amsmath} % assumes amsmath package installed
\usepackage{amssymb}  % assumes amsmath package installed
\usepackage{color}
\usepackage{graphicx}
\usepackage{subfigure}
\usepackage[english]{babel}
\usepackage{amsthm}
\usepackage{breqn}
\usepackage[super]{nth}
\usepackage{url}
\usepackage{siunitx}
\usepackage{makecell}
\usepackage{microtype}
\usepackage{booktabs}
\usepackage{hyperref}
\usepackage{stfloats}
\usepackage{threeparttable} % threeparttable
\usepackage{multirow}
\theoremstyle{definition}
\newtheorem{proposition}{Proposition}
\usepackage{verbatim}

\usepackage[utf8]{inputenc}
\usepackage{graphicx}
\usepackage{subfigure}
\usepackage{float}
\usepackage{bm}
\usepackage{amssymb}
\usepackage{threeparttable}
\usepackage{multirow}
\usepackage{eqnarray}
\usepackage{color}
\usepackage{graphicx}
\usepackage{tablefootnote}

\makeatletter
\newcommand{\removelatexerror}{\let\@latex@error\@gobble}
\makeatother

\SetKwInOut{KwParameter}{Parameters}

\title{\LARGE \bf
Contact-Aware Non-prehensile Robotic Manipulation for Object Retrieval in Cluttered Environments
}

\author{Yongpeng Jiang$^{*}$,
Yongyi Jia$^{*}$, and
Xiang Li
\thanks{$^{*}$ Equal contribution}
\thanks{
Y. Jiang, Y. Jia, and X. Li are with the Department of Automation, Tsinghua University, China. This work was supported in part by the National Natural Science Foundation of China under Grant U21A20517 and 52075290, and in part by the Science and Technology Innovation 2030-Key Project under Grant 2021ZD0201404. 
Corresponding author: Xiang Li (xiangli@tsinghua.edu.cn)
}}

\begin{document}

\maketitle
\thispagestyle{plain}
\pagestyle{plain}

\newtheorem{definition}{Definition}

%%%%%%%%%%%%%%%%%%%%%%%%%%%%%%%%%%%%%%%%%%%%%%%%%%%%%%%%%%%%%%%%%%%%%%%%%%%%%%%%
\begin{abstract}

Non-prehensile manipulation methods usually use a simple end effector, e.g., a single rod, to manipulate the object. Compared to the grasping method, such an end effector is compact and flexible, and hence it can perform tasks in a constrained workspace; As a trade-off, it has relatively few degrees of freedom (DoFs), resulting in an under-actuation problem with complex constraints for planning and control. 
%
% an advantage over grasping in constrained workspaces, the former is less limited by the manipulator's size and reachable area.
% The planning and control of non-prehensile manipulation in cluttered environments is challenging, because the imprecise contact modeling could seriously affect the algorithm's robustness. Although existing methods are able to avoid obstacles, we found it difficult to find a complete solution in cluttered settings.
%
% To deal with this problem, we proposed a new bilevel framework for non-prehensile manipulation, specifically, with a rod-like pusher.
This paper proposes a new non-prehensile manipulation method for the task of object retrieval in cluttered environments, using a rod-like pusher.
%
% The proposed method utilizes a rod-like pusher and has a bi-level structure. 
%
Specifically, a candidate trajectory in a cluttered environment is first generated
% at high-level 
% given the object's start and goal poses
with an improved Rapidly-Exploring Random Tree (RRT) planner; Then, a Model Predictive Control (MPC) scheme is applied to stabilize the slider's poses through necessary contact with obstacles.
Different from existing methods, the proposed approach is with the contact-aware feature, which enables the synthesized effect of active removal of obstacles, avoidance behavior, and switching contact face for improved dexterity. Hence both the feasibility and efficiency of the task are greatly promoted.
% actively utilizing the contacts with the environment to amplify the pushing path 
% and hence improve feasibility and efficiency.
%
% Different from existing methods, the proposed approach is with the contact-aware feature, 
%
% The efficiency is further enhanced with appropriate simplification of the system model at different levels.
%
%The stability and convergence of the proposed method are rigorously analyzed with the consideration of imprecise contact model, and its performance is validated in a planar object retrieval task, with a rod-like robot pusher, 
%
The performance of the proposed method is validated in a planar object retrieval task, where the target object, surrounded by many fixed or movable obstacles, is manipulated and isolated.
Both simulation and experimental results are presented.
%
% Simulation results show our methods achieve higher success rates than pure RRT approaches and pushing controllers which ignore contacts.
%
% Finally, we validate the framework with the Franka Panda robot in real experiments.

\end{abstract}

%%%%%%%%%%%%%%%%%%%%%%%%%%%%%%%%%%%%%%%%%%%%%%%%%%%%%%%%%%%%%%%%%%%%%%%%%%%%%%%%
\section{INTRODUCTION}
Manipulation in clutter is a skill commonly demanded in daily life and production, such as desktop arrangement 
%\cite{??}
and tidying up open shelves.
%\cite{??}.
%\tcolor{- use the examples here -  For example, consider the densely stacked cabinet where grasping from above is impractical, it is possible to push the target bottle from the side. Another example is the potential to move the heavy and fragile fish tank by sliding.}
%
Such a task is challenging for a robot manipulator because 
%the objects being manipulated differ greatly in shape and texture, and hence the contact model is usually unknown;
the dexterity of the robot end effector is often restricted by the cluttered environment and unknown object properties. For example, the stable grasp pose might be occluded by surrounding obstacles;
% \textcolor{red}{or the fragile and heavy object (i.e., a fish tank) is likely dangerous to lift up.}
%Thus the advantages of more flexible and versatile means of manipulation are highlighted.
%
% \textcolor{red}
{Another example is that fragile or heavy objects are generally dangerous to lift up.}
Non-prehensile manipulation proposed by Mason \cite{Mason1999ProgressIN} only requires no penetration constraints and does not rely on stable grasping \cite{Siciliano2018NonprehensileD}, which is suitable for performing tasks in cluttered environments.
%
% which therefore make the aforementioned tasks more tractable.
%
%
% Despite these advantages, it is still quite challenging to perform non-prehensile manipulation in confined spaces or contact-rich cases, which is inevitable for manipulation in clutter.

%
% As a trade-off, the non-prehensile end effector is with relatively few DoFs, resulting in a highly under-actuated problem with multiple constraints for obstacle avoidance, which opens up the challenges for the development of planning and control algorithm.  
%

This paper considers a representative and illustrative scenario in the problem of non-prehensile manipulation, that is, retrieving a target object
from clutter with a single rod-like pusher overhead, as seen in Fig. \ref{fig: teaser}. To achieve it, the pusher should contact and move the object (i.e., the planar slider) to the goal location in the presence of multiple obstacles. Such a task is not trivial, and the challenges can be summarized from the following aspects.
\begin{enumerate}
    \item[-] The pusher and slider correlated by frictional contacts form an underactuated system with hybrid dynamics (i.e., alterable contact faces and modes), %and hence some motion (moving backward) and poses have to be realized via detaching and re-contact.
    thus imposing complex kinodynamic constraints on planning and control.
    \item[-] Push planning in cluttered environments is limited by the widely known narrow corridor problem, which seriously restricts solving efficiency.
    % \item[-] Unintentional interactions among the slider and surrounding should be minimized; otherwise, unpredictable dangers (e.g., the turning over of bottles, the breaking apart of stacked cans) might occur consequently.  
    % \item[-] The mutual effects among mass of objects (e.g., contact model) are inherently hard to describe and forecast \cite{Huang2021VisualFT}.
\end{enumerate}

To address the problems above, existing methods 
%on planar non-prehensile manipulation planning 
add extra constraints to reduce the search space, such as demanding the contact mode to be consistent \cite{Doshi2019HybridDD} or limiting the slider's movement to a particular pattern (i.e., Dubins path) \cite{Zhou2019PushingR}.
However, such methods fail to sufficiently explore the state space, which might affect the solution quality. 
% thus having an adverse impact on solution quality.
%
Besides, most existing works consider avoidance of simple obstacles \cite{Moura2021NonprehensilePM} or implicitly assume an open space is required \cite{Doshi2019HybridDD, Chai2022ObjectRT}.
However, in cluttered environments, there might be no feasible path to the goal position if the manipulated object merely avoids obstacles, or the total efficiency is unacceptable as it might take a long time to complete all the avoidance. 
\begin{figure}[t]
    \centering
    \vspace{10pt}
    \includegraphics[width=.98\columnwidth]{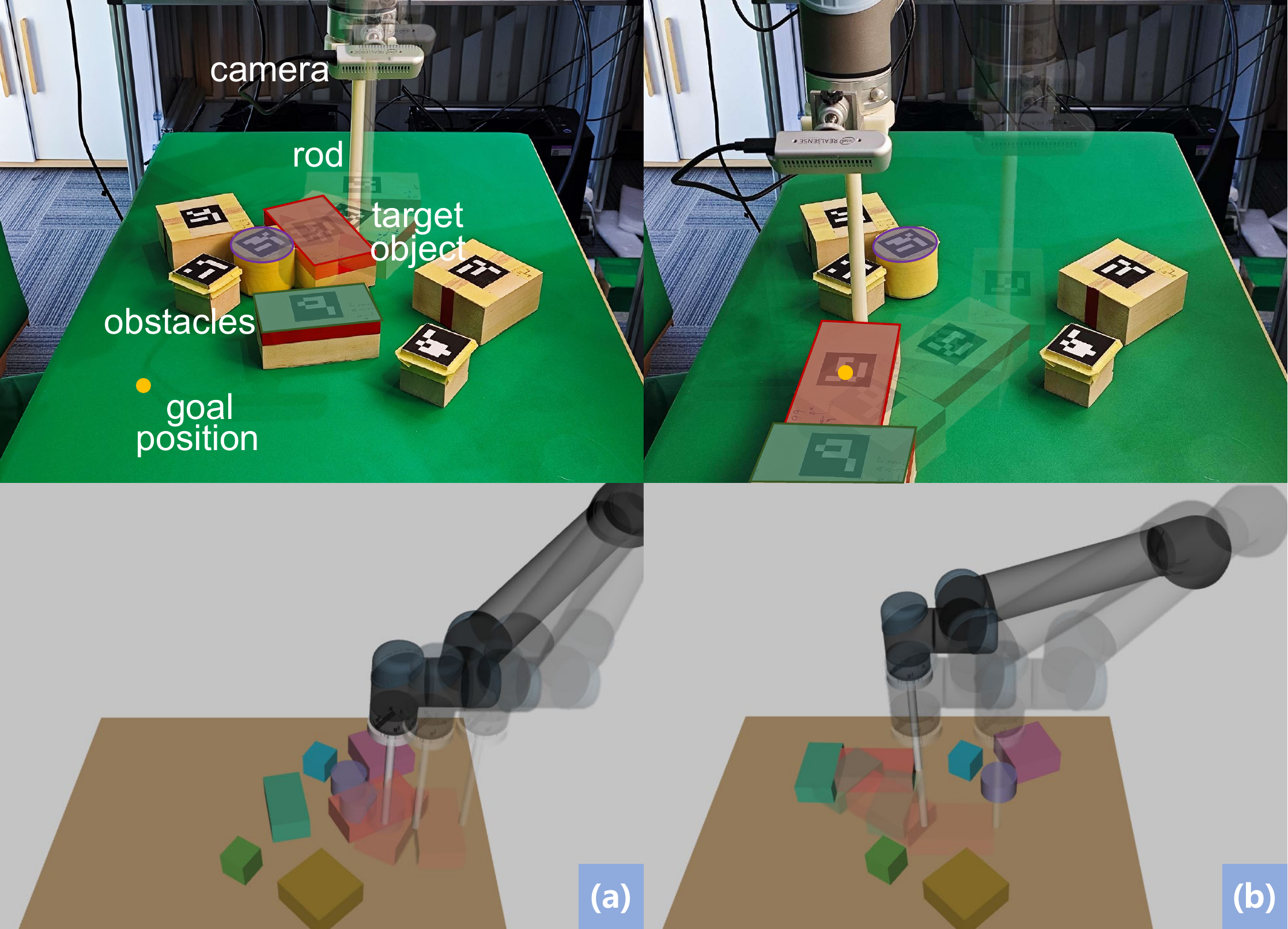}
    \caption{
        \textbf{Object retrieval task through planar pushing.}
        The target object in \textbf{\textcolor[RGB]{242,121,112}{red}} is separated from the clutter with a rod-like pusher through pushing manipulation. 
        \textbf{Top}: camera view. 
        \textbf{Bottom}: 3D visualization.
        \textbf{Left}: The \textbf{\textcolor[RGB]{137,131,191}{purple}} cylinder is pushed aside and thus the target object quickly navigates through fixed obstacles on both sides.
        \textbf{Right}: the \textbf{\textcolor[RGB]{50,184,151}{cyan}} cube is pushed away and thus the target object bursts through the clutter.
    }
    \label{fig: teaser}
    \vspace{-15pt}
\end{figure}
%
% \begin{figure*}[t]
%     \centering
%     \includegraphics[width=1.8\columnwidth]{figs/pickpose.pdf}
%     \caption{\textbf{The three industrial parts and their assigned graspable poses in our simulation.}
%     %
%     To avoid interfering with subsequent work processes , the end-effector can only suck each industrial part at the only two poses shown above. 
%     \textbf{Left}: object 0. 
%     \textbf{Middle}: object 1. 
%     \textbf{Right}: object 2. 
%     }
%     \vspace{-15pt}
%     \label{fig:pickpose}
% \end{figure*}

% \begin{figure}[t]
%    \centering
%    \includegraphics[width=0.8\columnwidth]{figs/teaser.pdf}
%    \caption{Caption}
%    \label{fig:teaser}
% \end{figure}

% To further improve the performance of TAMP, this paper proposes a new bilevel structure, which guarantees the feasibility and efficiency of non-prehensile manipulation in clutter. Specifically, 
%
% The main contribution of this paper is the development of a new bilevel framework for non-prehensile pushing manipulation in clutter. 
To improve the feasibility and efficiency of object retrieval in cluttered environments, this paper proposes a contact-aware non-prehensile manipulation method using the rod-like pusher, which integrates multiple actions of active removal of obstacles, avoidance behavior, and switching contact face to
% can actively push movable objects away and hence 
create a feasible path to the goal position if it is not available at the beginning. 
% while avoiding workspace obstacles.
%
% The proposed method is organized under a two-level structure as follows.
%
The proposed method is organized as follows.
\begin{itemize}
    \item[-] For global motion planning, 
    % only simplified dynamics is considered, including a dubin's curve planner with sticking contact and a rough object interaction model. A RRT planner with prioritized sampling strategy is designed for completeness and efficiency.
    an RRT planner guided by reachable sets is proposed to generate candidate trajectories for the target object and enables flexible choice of contact faces and modes for improved dexterity.
    \item[-] For local motion planning, an interaction model with necessary simplifications is applied to predict the outcome of the contacts between objects and then generate the safe and reliable active pushing action.
    \item[-] For motion and interaction control, 
    % we assume at most two objects are in contact via MPC, which ensures robust tracking of target object's trajectory.
    an MPC scheme stabilizes the object's pose around the candidate trajectory while navigating towards the goal, even during contact with obstacles. 
    % The robustness of such a method is validated through both simulation and real-world experiments.
\end{itemize}

Such a contact-aware feature allows the robot to explore different actions to generate more opportunities in cluttered environments.
% \textcolor{red}{- to add somathing about theoretical analysis -}
%
% Note that
% the proposed approach considers complete dynamics during the planning phase, such that the obtained trajectory is more stable; 
% Hence the obtained trajectory is more compatible with the controller. 
% acquired from the nested interaction model enables actively making contact with the environment to create or amplify the pushing path, 
% and hence improves feasibility and efficiency.
%
{Moreover, the simplification of pushing dynamics yields reachable sets and object interaction model, which efficiently guide motion planning.}
Simulation results and further robot experiments are presented to verify the effectiveness of the proposed method. 
\section{RELATED WORKS}
\subsection{Non-prehensile Pushing Manipulation}
% Non-prehensile manipulation was proposed by Mason as manipulation without grasping \cite{Mason1999ProgressIN}, and was revisited by Siciliano, who noted its advantages to make the manipulator less restricted in design, workspace and degrees of freedom \cite{Siciliano2018NonprehensileD}, especially in cluttered environments.
%
% Nevertheless, planning and control for non-prehensile manipulation are challenging, due to the hybrid dynamics, under-actuation and modeling uncertainties through making and breaking contacts \cite{Lynch2017PnCforDyn, Hogan2016FeedbackCO}.

Planar pushing \cite{Mason1982ManipulatorG} has recently become a representative task for developing non-prehensile manipulation algorithms. Modeling the pushing task needs to deal with the full relation between slide motion and frictional load.
% To achieve it, 
Several contributions in this field are summarized below, which have made the task increasingly tractable. Mason made the widely used quasi-static assumption \cite{Mason1982ManipulatorG};
Goyal \textit{et al.} \cite{Goyal1991PlanarS} proposed the limit surface force-motion model of the sliding system;
%which could be identified with the data-driven method developed by Zhou \textit{et al.} \cite{Zhou2016ConvexP}
Cutkosky \textit{et al.} \cite{Cutkosky1991FixureP} simplified the model by an ellipsoidal approximation; Zhou \textit{et al.} \cite{Zhou2019PushingR} validated the system's differential flatness properties. 
%Thus the efficiency of planning and control is guaranteed.
%

%A direct approach is to model the problem as a Mixed Integer Programming (MIP) \cite{Hogan2016FeedbackCO}, which represents the contact modes as integer variables. Hogan \textit{et al.} developed a prediction model based on Gaussian Process \cite{Hogan2018ADA} and neural network classifiers \cite{Hogan2020ReactivePN} to eliminate the non-convex integer variables.

In parallel, several works have also been proposed to deal with the control of discontinuous pushing dynamic caused by mutable contact modes (i.e. sticking, sliding).
Inspired by recent developments in contact-rich locomotion, frictional contacts could be efficiently handled by complementary constraints \cite{Zhou-RSS-17}. Moura \textit{et al.} \cite{Moura2021NonprehensilePM} applied the technique and built an MPCC-based framework for robust planning and control. Wang \textit{et al.} \cite{Wang2022ContactImplicitPA} used an alternative approach of State-Triggered Constraints (STC) and achieved better trajectory tracking performance. 

Most of the existing methods assume that the contact between the pusher and the slider is rigid and consistent
% restricting the controllability 
\cite{Xue2022DemonstrationGOC}; 
%Handling the sliding system generally faces 
While switching the contact positions makes it flexible to manipulate the object, the planning of such a formulation becomes more challenging as it is involved both the discrete contact and the continuous trajectory; 
To handle contact switches, Doshi \textit{et al.} \cite{Doshi2019HybridDD} proposed an exhaustive tree search method, Xue \textit{et al.} \cite{Xue2022DemonstrationGOC} used the human demonstration to guide the planning. 
However, these methods are limited either in the number of contact switches or in generalization ability, which does not function in densely cluttered environments, especially when the pushing trajectory is infeasible at the beginning if obstacles are not moved.
%Thus these methods are prone to failure in cluttered environments.
% Thus these methods are prone to failure in the cluttered object retrieval settings we considered.
%

%
% Our approach is similar to \cite{Moura2021NonprehensilePM}. Since intentional contacts are generally helpful to push objects out of clutter, i.e., pushing to clear away the obstacles, our method considers motion stabilization in contact with the environment. 
% Pan \textit{et al.} \cite{Pan2020DecisionMI} considered object interaction by design, but their algorithm applies to large-scale pushing manipulation, which is imprecise for our task.
%
% Another work similar to ours is \cite{Chai2020AdaptiveUO}, we both apply the RRT algorithm for sequential planning, but their approach is designed for open space manipulation.
%

\subsection{Manipulation Planning in Cluttered Environments}
% \subsection{Kinodynamic Motion Planning}
%
% Sampling-based planners, such as RRTs have been widely applied in dexterous manipulation \cite{Cheng2020ContactMG, Cheng2021ContactMG, Pasricha2022PokeRRTPA}, due to its advantages in efficiently exploring high-dimensional workspace and handling holonomic constraints.
%
% Most existing works on planar pushing implicitly assume an open space is required \cite{Chai2022ObjectRT}. Nonetheless, cluttered settings are more common in household, logistics and manufacture domains, such as object retrieval \cite{Huang2021VisualFT, Papallas2019NonPrehensileMI, Zeng2018LearningSB}, arrangement \cite{Anders2018ReliablyAO} and sorting \cite{Pan2020DecisionMI}.
%
% We examine the object retrieval task. While it is common to expect a time optimal solution \cite{Papallas2019NonPrehensileMI} or minimizing the number of push actions \cite{Huang2021VisualFT}, we focus on maintaining feasibility and efficiency on the premise of minimum contact with obstacles, so as to avoid tipping objects over and unexpected emergency occurring.
%
% Highly unstructured search space is one of the main bottlenecks for planning in clutter.
%
Existing works apply hierarchical planning, exhaustive tree search, and sampling-based approaches for manipulation planning in cluttered environments. 

First, hierarchical approaches are introduced to break the task down into several steps; Gao \textit{et al.} \cite{Gao2021FastHT} and Nam \textit{et al.} \cite{Nam2021FastAR} solved the object rearrangement task by solving combinatorial optimization at high level and motion planning at low level; These partitions designed for prehensile manipulation may not be suitable for complex kinodynamic constraints in pushing. Second, %
tree search can handle prehensile and non-prehensile cases \cite{Doshi2019HybridDD, Song2019MultiObjectRW, Chen2021TrajectoTreeTO}, but it is subject to the combination explosion and relies on task-specific heuristics for pruning. Third, 
sampling-based techniques are less sensitive to 
%problem scale
the branching factor and enable 
%rapid exploration over complex spaces
efficiently exploring high-dimensional search space and handling holonomic constraints \cite{Cheng2020ContactMG};
However, such methods commonly need expensive tree extension processes for non-holonomic dynamics \cite{Webb2013KinodynamicRA}, for instance, the pushing system (which is due to solving non-trivial two-point Boundary Value Problems (BVPs) when connecting arbitrary states \cite{Xie2015TowardAO}); %
To tackle the problem, Webb \textit{et al.} \cite{Webb2013KinodynamicRA} proposed Kinodynamic RRT*, which uses an optimal controller for state connection; 
% but their approach is restricted to linear dynamics. 
For general nonlinear systems, Goretkin \textit{et al.} \cite{Goretkin2013OptimalSP} applied local approximation on dynamic constraints and cost function.

Different from the above approaches which
assume continuous dynamics, the pushing system includes hybrid dynamics. Our work is inspired by utilizing reachable sets to guide the tree expansion \cite{Shkolnik2009ReachabilityguidedSF,Wu2020R3TRR} and further extends the approach to
pushing manipulation.

In summary, existing methods for manipulation planning in cluttered environments are commonly limited by several open issues: hybrid dynamics
%
% Different from the above approaches which assume the dynamics to be continuous, the pushing system we consider includes hybrid dynamics, 
due to contact face switching and contact mode scheduling \cite{Hogan2016FeedbackCO}, pushing manipulation with fewer DoFs \cite{Wei2017ManipulatorMP}, obstacles with irregular shapes \cite{Moura2021NonprehensilePM,Chai2020AdaptiveUO}, and heavy physics engines and large datasets \cite{Wei2017ManipulatorMP}. Those issues will be systematically addressed in this paper.

\section{METHOD}%\section{BACKGROUND AND PROBLEM STATEMENT}
In this section, we formally define the planar object retrieval task and propose the contact-aware planning algorithm in Sec.~\ref{sec: problem formulation}.
% We  further introduce the computation of reachable sets that direct the search in Sec.~\ref{sec: computation of reachable sets}. 
Implementation of the contact-aware feature requires the computation of reachable sets for guided search (Sec.~\ref{sec: computation of reachable sets}), the object interaction module for motion prediction (Sec.~\ref{sec: object interaction model}), and a robust feedback controller for trajectory tracking (Sec.~\ref{sec: model predictive pushing controller}).
% Then, we present the object interaction model for motion prediction in Sec.~\ref{sec: object interaction model}.
% Finally, we propose the contact-aware planning algorithm for the planar object retrieval task in Sec.~\ref{sec: planar pushing planning},
% Finally we present the pushing controller in Sec.~\ref{sec: model predictive pushing controller}.
The block diagram of the proposed framework is depicted in Fig. \ref{fig: algorithm framework}.

\begin{figure*}[tbp]
    \centering 
    \subfigcapskip=0pt
    \includegraphics[height=1.95in]{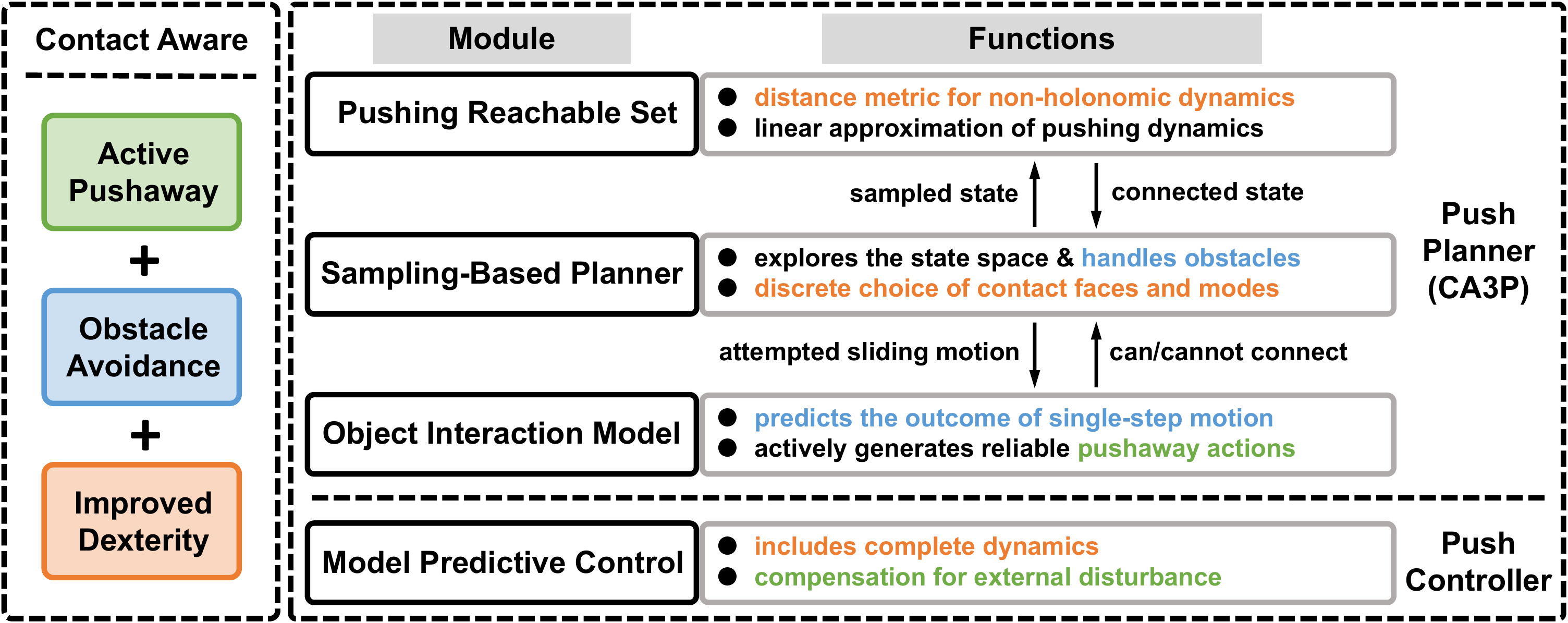}
    \vspace{-5pt}
    \caption{
        Block diagram of the proposed framework (the planner and controller). Elements and implementations of the contact-aware feature are highlighted.
    }
    \label{fig: algorithm framework}
    \vspace{-10pt}
\end{figure*}

\subsection{Pushing Task and Pushing Planner} \label{sec: problem formulation}

\subsubsection{Problem Formulation} \label{sec: retrieval problem formulation}
This paper considers a planar workspace 
% $\mathcal{X}$
with the target object (also referred to as the planar slider) $o^s$ and multiple fixed or movable obstacles denoted as $o_1^t,\dots,o_{\vert \mathcal{O}^t \vert}^t\in\mathcal{O}^t$ and $o_1^m,\dots,o_{\vert \mathcal{O}^m \vert}^m\in\mathcal{O}^m$, respectively. We assume the target object and obstacles are convex polygons with known geometric parameters and their initial poses 
% $\bm{x}^s[0],\bm{x}_1^t[0],\dots,\bm{x}_{\vert \mathcal{O}^t \vert}^t[0],\bm{x}_1^m[0],\dots,\bm{x}_{\vert \mathcal{O}^m \vert}^m[0] \in SE(3)$
are also given. The target object is actuated by a single rod-like pusher $p$ through frictional point contact, with friction coefficient $\mu_p$. A graphical representation of the task is shown in Fig. ~\ref{fig: retrieval task definition}. Note that the retrieval task in most of the existing works is to solve the control sequence which drives the target object from initial pose $\bm{x}^s[0]$ to goal region $\mathcal{X}_g^s \subset \mathcal{X}^s$, without collision with all obstacles $\mathcal{O}^t \cup \mathcal{O}^m$. 

This paper additionally considers actively utilizing contacts with the environment to create or amplify the pushing path. Hence, we consider solving the task in the joint state space $\mathcal{X}^E = \mathcal{X}^s \times \mathcal{X}_1^m \times \cdots \times \mathcal{X}_{\vert \mathcal{O}^m \vert}^m$, where $\times$ denotes the Cartesian product. A state $\bm{x}^E=(\bm{x}^s, \bm{x}_1^m,\dots,\bm{x}_{\vert \mathcal{O}^m \vert}^m)$ is called feasible if the fixed obstacles are not in contact with other objects (i.e., in case of turning over objects or getting stuck). The state transition is regulated by pushing dynamics (\ref{sec: pushing dynamics model}) and the object interaction model (\ref{sec: object interaction model}). 

Hence, the objective of this paper is to find
% planar object retrieval task we considered is defined as finding 
the feasible state sequence $\bm{x}^s{[0:T]}$ and corresponding control sequence $\bm{u}^p{[0:T-1]}$ which drives the target object to goal region, $T$ is the path length. Note that the joint states $\bm{x}^E{[0:T]}$ which contain $\bm{x}^s{[0:T]}$ should all be feasible.
\begin{figure}[t]
    \centering
    \includegraphics[width=.95\columnwidth]{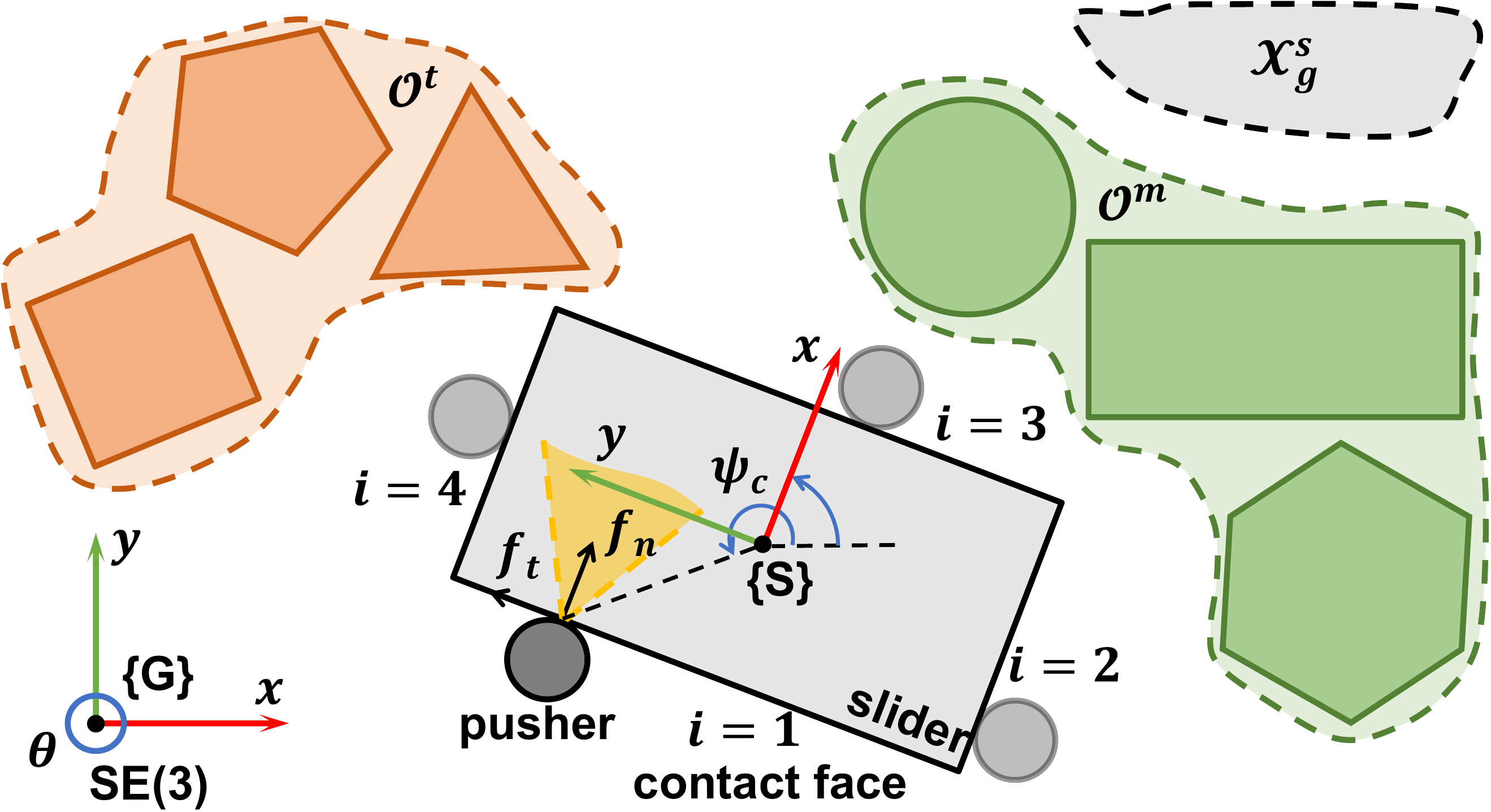}
    \caption{
        Graphical representation of the planar object retrieval task, including the target object (planar slider) painted in grey, the goal region $\mathcal{X}_g^s$, movable obstacles $\mathcal{O}^m$ and fixed obstacles $\mathcal{O}^t$, $i\in\left\{1,2,3,4\right\}$ denotes the discrete choice of contact face. Yellow shade represents the friction cone constraints.
    }
    \label{fig: retrieval task definition}
    \vspace{-20pt}
\end{figure}

\subsubsection{Contact-Aware Pushing Planner}\label{sec: contact-aware pushing planner}

The \textbf{C}ontact-\textbf{A}ware \textbf{P}lanar \textbf{P}ush \textbf{P}lanner (CA3P) is outlined in Algorithm \ref{alg: contact-aware planar push planar}. An illustration of the CA3P is given in Fig. ~\ref{fig: CA3P and reachable set} \textbf{Left}. Above all, the contact-aware feature could be divided into \textbf{three levels}. Specifically, the active avoidance and clearance of obstacles, and the contact face-switching technique to increase dexterity. The proposed algorithm is derived from the RRT planner. The search tree of RRT is stored in $\mathcal{H}^s$, each node 
% $\left(\bm{x}^s,\bm{u}^p\right)$
in $\mathcal{H}^s$ records a feasible pose of the target object and the corresponding control input. 
% and the corresponding movable obstacles $\mathcal{O}^m$ is recorded in $\mathcal{H}^m$.
To push aside movable obstacles to widen the pushing path and avoid other obstacles, we simultaneously track the changeable poses of $\mathcal{O}^m$ and the poses of $\mathcal{O}^t$ in $\mathcal{H}^o$, which is referred to as the \textbf{planning scene}. Note that $\mathcal{H}^s$ and $\mathcal{H}^o$ are identical in structure, which can be seen in Fig. \ref{fig: CA3P and reachable set} \textbf{Left}.

Since the pushing system is subject to non-holonomic constraints, the Euclidean distance works poorly; we adopt the concept of reachable sets \cite{Wu2020R3TRR} as a distance metric. In Algorithm \ref{alg: contact-aware planar push planar},
${\mathtt{NearestNeighbor}}$ computes the nearest state $\bm{x}_{\text{near}}$ in all reachable sets of states in $\mathcal{H}^s$ to a newly sampled state $\bm{x}_{\text{new}}$ (Line 5). 
We omit the superscript $s$ of the target object for brevity.
%We refer the readers to \cite{Wu2020R3TRR} for details. 
Then, ${\mathtt{GetGenerateState}}$ returns the \textbf{generating state} $\bm{x}_{\text{gen}}$ of the reachable set
% $\mathcal{R}_{\tau}(\bm{x}_{*}^t)$
which contains $\bm{x}_{\text{near}}$ (Line 6). The notations and computation of reachable sets will be presented in Sec.~\ref{sec: computation of reachable sets}. Next, ${\mathtt{Connect}}$ calculates the control input $\bm{u}$ driving the system from $\bm{x}_{\text{gen}}$ to $\bm{x}_{\text{near}}$ (Line 7). Due to the linearization error of pushing dynamics, the rollout of $\bm{u}$ usually does not reach $\bm{x}_{\text{near}}$ exactly; We denote the actual terminal state as $\bm{x}_{\text{term}}$. The state connection is realized by applying the discrete-time Linear-Quadratic Regulator (LQR) controller to the linear system:
\begin{equation}
    \bm{x}[k+1]={I}_4\bm{x}[k]+(\tau_{\text{LQR}} \bm{B})\bm{u}[k]
    \label{eq: state connection LQR system}
\end{equation}
with step size $\tau_{\text{LQR}}$, state cost $C_Q$ and input cost $C_R$. The input matrix $\bm{B}$ will be further defined in Sec.~\ref{sec: computation of reachable sets}, the subscript $i$ is also ignored for brevity.

The core procedure that enables the active clearance of obstacles is ${\mathtt{Simulate}}$ (Line 9), which predicts the outcome (i.e., new obstacle poses) of one-step target object motion. ${\mathtt{GetEnviron}}$ first queries the \textbf{planning scene} $\mathcal{O}_{\text{gen}}^m$ correlated with $\bm{x}_{\text{gen}}$ in $\mathcal{H}^o$ (Line 8). Next, the configuration of the movable obstacles 
% $\mathcal{O}_{*}^m$
is updated through the object interaction model in Sec.~\ref{sec: object interaction model}. The state connection is abandoned if collision with fixed obstacles is detected. If the state connection is successful, the search tree is updated with new \textbf{planning scene} ${\mathtt{Update}}(\mathcal{O}_{\text{gen}}^m)$.

The abovementioned procedure is repeated until the maximum number of nodes is exceeded. Finally, the control sequence $\bm{u}{[0:T-1]}$ is extracted from $\mathcal{H}^s$ if the goal region $\mathcal{X}_g$ is reachable from current states (Line 15).

%%%%%%%%%%%%%%%%%%%%%%%%%%%%%%%%%%%%%%%%%%%%%%%
%%%%%%%%%%%%%%%%% Pseudo Code %%%%%%%%%%%%%%%%%
%%%%%%%%%%%%%%%%%%%%%%%%%%%%%%%%%%%%%%%%%%%%%%%
%
\begin{figure}[!t]
    \removelatexerror
    \begin{algorithm}[H]
        \label{alg: contact-aware planar push planar}
        \caption{CA3P}
        \LinesNumbered
        \KwIn{State space $\mathcal{X}^s
        % ,\mathcal{X}^E
        $, obstacles $\mathcal{O}^t,\mathcal{O}^m$, initial pose $\bm{x}[0]$, goal region $\mathcal{X}_g$}
        % \KwOut{State sequence $\bm{x}_{[0:T]}^t$, control sequence $\bm{u}_{[0:T-1]}^p$}
        \KwOut{Control sequence $\bm{u}{[0:T-1]}$}
        \KwParameter{Reachable set step size $\tau$, LQR step size $\tau_{\text{LQR}}$, maximum number of nodes $n_{\text{max}}$}
        \While{\rm{\textbf{not}} maximum number of nodes exceeded}{
            Initialize search tree $\langle \mathcal{H}^s,\mathcal{H}^o \rangle$;
            % , \mathcal{H}^t\hspace{-0.2mm}\leftarrow\hspace{-0.2mm}\varnothing,\mathcal{H}^m\hspace{-0.2mm}\leftarrow\hspace{-0.2mm}\varnothing$;

            $\mathcal{H}^s\text{.nodes}{\rm{.add}}\left((\bm{x}[0],\varnothing)\right), \mathcal{H}^o\text{.nodes}{\rm{.add}}\left(\mathcal{O}^t\bigcup\mathcal{O}^m\right)$;
            
            Randomly sample $\bm{x}_{\text{new}} \in \mathcal{X}^s$;
            
            % $\mathcal{R}_{\tau}(\bm{x}_{*}^t) \leftarrow \mathop{\arg\min}\limits_{\bm{x}^t \in \mathcal{H}^t}{\rm {dist}}\left(\bm{x}_{\text{new}}^t,\mathcal{R}_{\tau}(\bm{x}^t)\right)$;

            $\bm{x}_{\text{near}}\hspace{-1.0mm}\leftarrow\hspace{-1.0mm}{\mathtt {NearestNeighbor}}\hspace{-1.0mm}\left(\hspace{-1.0mm}\bm{x}_{\text{new}},\hspace{-1.0mm}\bigcup\limits_{\bm{x} \in \mathcal{H}^s\text{.nodes}}\hspace{-1.0mm}\mathcal{R}_{\tau}(\bm{x})\hspace{-1.0mm}\right)$;

            $\bm{x}_{\text{gen}} \leftarrow {\mathtt{GetGenerateState}}(\bm{x}_{\text{near}})$;

            % $(\bm{x}_{1}^m,\dots,\bm{x}_{\vert\mathcal{O}^m\vert}^m)_{*}\hspace{-0.3mm}\leftarrow\hspace{-0.3mm}

            $\bm{x}_{\text{term}},\bm{u}{[0:\tau]} \leftarrow {\mathtt{Connect}}\left(\bm{x}_{\text{gen}},\bm{x}_{\text{near}}\right)$;

            $\mathcal{O}_{\text{gen}}^t\bigcup\mathcal{O}_{\text{gen}}^m \leftarrow {\mathtt{GetEnviron}}\left(\mathcal{H}^o\text{.nodes},\bm{x}_{\text{gen}}\right)$;
            
            \If{${\mathtt{Simulate}}(\bm{x}_{\text{\rm{gen}}},\bm{u}{[0:\tau]},\mathcal{O}_{\text{\rm{gen}}}^m,\mathcal{O}_{\text{\rm{gen}}}^t)$}{
                $\mathcal{H}^s\text{.nodes}{\rm{.add}}\left((\bm{x}_{\text{term}},\bm{u}{[0:\tau]})\right)$;

                $\mathcal{H}^s{\text{.edges}}{\rm{.add}}\left((\bm{x}_{\text{gen}},\bm{x}_{\text{term}})\right)$;

                $\mathcal{H}^o{\text{.nodes}}{\rm{.add}}\left(\mathcal{O}_{\text{gen}}^t \bigcup {\mathtt{Update}}(\mathcal{O}_{\text{gen}}^m)\right)$;

                $\mathcal{H}^o{\text{.edges}}{\rm{.add}}\left((\bm{x}_{\text{gen}},\bm{x}_{\text{term}})\right)$;
                
                \If{${\mathtt{Connect}}(\bm{x}_{\text{\rm{term}}},\mathcal{X}_g)$}{
                    \Return{${\mathtt{ExtractPath}}\left(\mathcal{X}_g,\mathcal{H}^s\right)$};
                }
            }
        }
    \Return{$\varnothing$};
    \end{algorithm}
    \vspace{-18pt}
\end{figure}
%

%
% \begin{figure*}[t]
%     \centering
%     \includegraphics[width=1.7\columnwidth]{figs/framework_small.pdf}
%     \vspace{-5pt}
%     \caption{\textbf{Framework of dynamic RL on pre-grasp nudging manipulation task.}
%     %
%     Different from the original RL framework, dynamic RL not only outputs actions, but also selects a suitable duration for each action.
%     %
%     The trajectory generator plans motions for the end-effector by considering actions and object poses together.
%     %
%     The motion of the end-effector is repeated according to the action duration given by dynamic RL.
%     %
%     The environment executes the repeated motion, and returns the latest object pose to dynamic RL and trajectory generator for the next action.
%     }
%     \label{fig:my_label}
%     \vspace{-15pt}
% \end{figure*}

\vspace{-5pt}
\subsection{Pushing Reachable Sets} \label{sec: computation of reachable sets}

In this subsection, we first introduce the dynamic model and constraints of planar pushing and then present the computation of reachable sets, including an approximation to the dynamics.

\vspace{5pt}
\subsubsection{Pushing Dynamics} \label{sec: pushing dynamics model}
% This section defines the state, input variables, and the nonlinear dynamic model of the pushing system. 
As seen in (\ref{fig: retrieval task definition}), state variables of the dynamic model can be chosen as the configuration of the planar slider and pusher
\begin{equation}
    \bm{x}^s\triangleq\left[x^s,y^s,\theta^s,\psi_c\right]^T
    \label{eq: pushing state variables},
\end{equation}
where $x^s,y^s, \theta^s$ denote the position and the orientation of the slider (i.e., target object) with respect to the global frame, and $\psi_c\in\left[-\pi,\pi\right]$ is the azimuth angle of the pusher's contact point on the slider's periphery. 

Input variables include the pusher's contact force and velocity
\begin{equation}
    \bm{u}^p\triangleq\left[f_n,f_t,\dot{\psi_c}\right]^T
    \label{eq: pushing input variables},
\end{equation}
where $f_n,f_t$ represent the normal and tangential forces at the slider's local coordinates. The concept of limit surface proposed by Goyal \textit{et al.} \cite{Goyal1991PlanarS} maps the frictional wrench $\mathcal{F}^s$ to the slider's twist $\mathcal{V}^s$: $\mathcal{V}^s=\bm{A}\mathcal{F}^s$, where $\bm{A}$ is a positive definite matrix defined by the maximum frictional wrench that can be exerted on the slider. 

Under quasi-static assumptions, the system dynamics model can now be given as
\begin{equation}
    \dot{\bm x}^s=\bm{f}_i\left(\bm {x}^s,\bm {u}^p\right)\triangleq\left[\begin{array}{cc}
        \bm R \bm A \bm J_{c,i}^\top & 0 \\
        0 & 1
    \end{array}\right]\bm{u}^p
    \label{eq: pushing dynamics},
\end{equation}
where $i=1,\dots,N$, $N$ is the number of contact faces, $\bm{R}$ is the rotation matrix from the slider's local coordinates to global coordinates, and $\bm{J}_{c,i}$ is the contact jacobian. The subscript $i$ is introduced to denote the discontinuous dynamics due to switching contact face, as illustrated in Fig. \ref{fig: retrieval task definition}. 
% We refer the readers to \citemaps the991PlanarS} for details.

In practice, the state variables are constrained in the workspace $\bm{x}^s \in \mathcal{X}^s$, and the input variables are subject to box constraints and Coulomb friction constraints. We define $\mathcal{U}_f:0\leq f_n \leq \bar{f}$, $\mathcal{U}_{\psi+}: \dot{\psi}_c>0$, $\mathcal{U}_{\psi-}:\dot{\psi}_c<0$ such that
\begin{equation}
    \mathcal{U}_{st}:\left\{
    \begin{aligned}
    & \mathcal{U}_f \\
    & \vert f_t \vert \leq \mu_p f_n \\
    & \dot{\psi_c}= 0
    \end{aligned}
    \right.,\,
    \mathcal{U}_{sl}:\left\{
    \begin{aligned}
    & \mathcal{U}_f, \mathcal{U}_{\psi-} \\
    & f_t = \mu_p f_n \\
    & -\bar{\dot{{\psi}_c}} \leq \dot{\psi_c}
    \end{aligned}
    \right.,
    \mathcal{U}_{sr}:\left\{
    \begin{aligned}
    & \mathcal{U}_f, \mathcal{U}_{\psi+} \\
    & f_t = -\mu_p f_n \\
    & \dot{\psi_c} \leq \bar{\dot{{\psi}_c}}
    \end{aligned}
    \right.\hspace{-0.05cm},
    \label{eq: pushing input constraints}
\end{equation}
where $\bar{f}, \bar{\dot{\psi}}_c$ are upper bounds of the pusher's contact force and velocity, $\mathcal{U}_{st}$, $\mathcal{U}_{sl}$ and $\mathcal{U}_{sr}$ represent the input constraints for sticking, sliding left and sliding right contact modes respectively, and $\mu_p$ is the friction coefficient between pusher and slider.

\begin{figure}[t]
    \centering
    \includegraphics[width=.95\columnwidth]{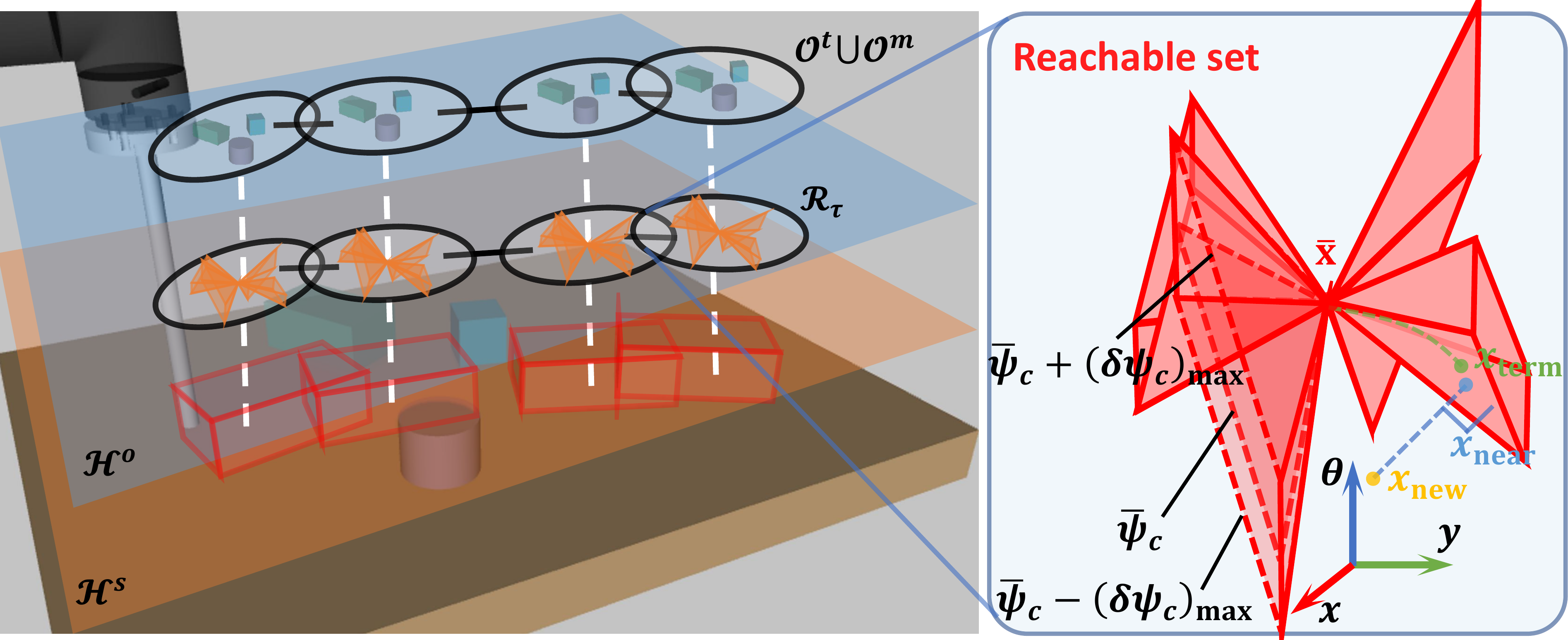}
    \caption{
        Illustration of the CA3P search trees and the reachable set. 
        \textbf{Left}: reachable sets $\mathcal{R}_{\tau}$ and planning scenes $\mathcal{O}^t\bigcup\mathcal{O}^m$ are stored in trees of the same structure, denoted as $\mathcal{H}^s$ and $\mathcal{H}^o$.
        \textbf{Right}: the reachable set is composed of several convex cones; each cone corresponds to pushing a certain contact face. Each \textbf{\textit{slice}} of the cone is with a fixed contact point location, e.g.,  ${\bar{\psi}}_c-(\delta\psi_c)_{max},\bar{\psi}_c,{\bar{\psi}}_c+(\delta\psi_c)_{max}$. Generating state $\bar{\bm{x}}$ of the reachable set, the sampled state $\bm{x}_{\text{new}}$, nearest neighbor $\bm{x}_{\text{near}}$, and terminal state $\bm{x}_{\text{term}}$ reached by state connection are shown.
    }
    \label{fig: CA3P and reachable set}
    \vspace{-18pt}
\end{figure}

\subsubsection{Computation of Reachable Sets} \label{sec: computation of reachable sets}
The reachable sets highlight the states more likely to be connected from the already explored state space. This technique can provide directional guidance to kinodynamic push planning. The reachable sets of arbitrary state $\bar{\bm{x}} \in \mathcal{X}$ is defined as the set of states reachable from $\bar{\bm{x}}$ within finite time horizon $\tau$, under the dynamic constraints and constraints on state and input variables:
\begin{equation}
    \begin{aligned}
        \mathcal{R}_{\tau}(\bar{\bm{x}})\triangleq
        % \left\{
        \{
            \bm{x} \in \mathcal{X} \vert &\exists(x,u):[0,t]\mapsto(\mathcal{X},\mathcal{U}),t\in[0,\tau],
            % \right.
            \\
            % \left.
            &\hspace{-0.5cm} x(0)=\bar{\bm{x}},x(\tau)={\bm{x}},\dot{x}(\xi)=\bm{f}_i\left(x(\xi),u(\xi)\right)
        % \right\}
        \}
    \end{aligned}
    \label{eq: reachable set definition}.
\end{equation}
We call $\bar{\bm{x}}$ the \textbf{generating state} of the \textbf{reachable set} $\mathcal{R}_{\tau}(\bar{\bm{x}})$. Using the time integration of (\ref{eq: pushing dynamics}), the terminal state is computed as
\begin{equation}
    x(t)=\bar{\bm{x}}+\int_{0}^{t}\bm{f}_i(x(\xi),u(\xi))d\xi
    \label{eq: terminal state computation}
\end{equation}
Note that due to nonlinear pushing dynamics (\ref{eq: pushing dynamics}), the analytic representation of (\ref{eq: terminal state computation}) is hard to obtain, and so is (\ref{eq: reachable set definition}).

Assuming $\tau$ is small, it is reasonable to make a linear approximation of (\ref{eq: pushing dynamics}) at 
% nominal state
state $\bar{\bm{x}}$ and input $\bar{\bm{u}}$ as
\begin{equation}
    \bm{f}_i({\bm{x}},{\bm{u}}) \approx \bm{f}_i(\bar{\bm{x}},\bar{\bm{u}})+\bm{A}_i({\bm{x}}-\bar{\bm{x}})+\bm{B}_i({\bm{u}}-\bar{\bm{u}})
    \label{eq: linearize dynamics},
\end{equation}
where $\bm{A}_i=\frac{\partial}{\partial {\bm{x}}}\bm{f}_i({\bm{x}},{\bm{u}})\Big|_{(\bar{\bm{x}},\bar{\bm{u}})}$, $\bm{B}_i=\frac{\partial}{\partial \bm{u}}\bm{f}_i(\bm{x},\bm{u})\Big|_{(\bar{\bm{x}},\bar{\bm{u}})}$, note that $\bm{A}_i=0$ for $\bar{\bm{u}}=0$.
% , and note that $A_i=\Large 0$ holds for the pushing dynamics.

In addition, we rewrite the input constraints in (\ref{eq: pushing input constraints}) as linear inequalities:
\begin{equation}
    \mathcal{U}_j=\left\{\bm{u}\vert\bm{D}_j{\bm{u}} \leq \bm{h}_j\right\}, j\in\{st,sl,sr\}
    \label{eq: pushing input constraint linear equality},
\end{equation}
and assume ${\bm{u}}$ is invariant over the considered time horizon $\tau$, $\bar{\bm{u}}=0$, thus an approximation to the \textbf{terminal states} (\ref{eq: terminal state computation}), corresponding to the $i^{\text{th}}$ contact face and the $j^{\text{th}}$ contact mode, under all possible inputs are formulated as a set:
% \begin{equation}
%     \begin{aligned}
%         \mathcal{T}_{\tau}(\bar{\bm{x}},\bar{\bm{u}})_{ij}=
%         % \left\{
%         \{
%             {\bm{x}}
%             \vert &
%             {\bm{x}}=\bar{\bm{x}}+\tau
%             \left[
%                 \bm{f}_i(\bar{\bm{x}},\bar{\bm{u}})+\bm{B}_i({\bm{u}}-\bar{\bm{u}})
%             \right],
%             % \right.
%             \\
%             % \left.
%             % & \bm{D}_j{\bm{u}} \leq \bm{h}_j
%             & \bm{u}\in\mathcal{U}_j
%         % \right
%         \}
%     \end{aligned}
%     \label{eq: terminal state approximation}.
% \end{equation}
\begin{equation}
    \begin{aligned}
        \mathcal{T}_{\tau}(\bar{\bm{x}})_{ij}=
        \left\{
            {\bm{x}}
            \vert
            {\bm{x}}=\bar{\bm{x}}+\tau
            \bm{B}_i{\bm{u}},
            \bm{u}\in\mathcal{U}_j
        \right\}
    \end{aligned}
    \label{eq: terminal state approximation}.
\end{equation}

Since we apply the linearized dynamics of (\ref{eq: pushing dynamics}), $x$, $y$ and $\theta$ are independent of $\dot\psi_c$, the change of contact location $\psi_c$ 
% is excluded from state variables in (\ref{eq: reachable set definition})$\sim$(\ref{eq: terminal state approximation})
does not affect the $(x,y,\theta)$ dimensions of the reachable set, leading to an over-conservative approximation. To consider the influence of the pusher's movement on the slider's periphery, we consider all possible contact point locations during time horizon $\tau$, that is, $\psi_c \in 
\mathcal{P}=
[{\bar{\psi}}_c-(\delta\psi_c)_{max},{\bar{\psi}}_c+(\delta\psi_c)_{max}]$, where $(\delta\psi_c)_{max}=\tau
% \dot\psi_{c,m}
\bar{\dot{\psi}}_c$. Hence an expansion
% set of terminal states based on
of ($\ref{eq: terminal state approximation}$) is of the following form:
\begin{equation}
    \mathcal{AT}_{\tau}(\bar{\bm{x}})_{ij}=\bigcup_{\psi_c\in\mathcal{P}}
    % \mathcal{T}_{\tau}(\left[\bar{x},\bar{y},\bar{\theta},\psi_c\right]^T,\bar{\rm u})
    % \mathcal{T}_{\tau}\left(\bar{\bm{x}}\right)_{ij}
    \mathcal{T}_{\tau}\left(\left[\bar{x},\bar{y},\bar{\theta},\psi_c\right]\right)_{ij}
    \label{eq: augmented terminal state approximation}.
\end{equation}
\begin{proposition}
    \textit{$\forall \bar{\bm{x}}\in\mathcal{X},\bar{\bm{u}}=0,\tau \geq 0$, the expanded set of terminal states $\mathcal{AT}$ in (\ref{eq: augmented terminal state approximation}) is convex.}
\end{proposition}
%The proof is quite intuitive. 
%%%%%%%%%%%%%%%%%%%%%%%%%%%%%%%%%%%%%%%%%%%%%%%%%%%%%%%%%%%%%
%%%%%%%%%%%%%%%%%%%% OLD VERSION - WRONG %%%%%%%%%%%%%%%%%%%%
%%%%%%%%%%%%%%%%%%%%%%%%%%%%%%%%%%%%%%%%%%%%%%%%%%%%%%%%%%%%%
% \begin{proof}
% We denote the elements in (\ref{eq: augmented terminal state approximation}) as ${\rm x}=\left[x^t,y^t,\theta^t\right]^T$.
% % Due to special properties of the matrix $B_i$, 
% %which can be easily derived from (\ref{eq: pushing dynamics}), 
% Since only the third column of $\bm{J}_{c,i}$ is dependent on the pusher's contact location, 
% only $\theta^t$ is relative to $\psi_c$.
% Hence $\mathcal{AT}_{\tau}(\bar{\rm x},\bar{\rm u})_{ij}$ can be viewed as the Cartesian product $\mathcal{S}_1\times\mathcal{S}_2$, where $(x^t,y^t)\in\mathcal{S}_1$ and $\theta^t\in\mathcal{S}_2$. The set in (\ref{eq: terminal state approximation}) is an affine transformation of the convex set $\left\{{\rm u} \vert D_j {\rm u} \leq h_j \right\}$, hence is convex, and the component $(x^t,y^t)$ forms the convex set $\mathcal{S}_1$. Besides, $\theta(\psi_c)$ is continuous on $\psi_c$, thus $\mathcal{S}_2$ is convex. In conclusion, the set of terminal states $\mathcal{AT}_{\tau}(\bar{\rm x},\bar{\rm u})_{ij}$ is convex.
% \end{proof}
%%%%%%%%%%%%%%%%%%%%%%%%%%%%%%%%%%%%%%%%%%%%%%%%%%%%%%%%%%%%%
%%%%%%%%%%%%%%%%%%%% OLD VERSION - RIGHT %%%%%%%%%%%%%%%%%%%%
%%%%%%%%%%%%%%%%%%%%%%%%%%%%%%%%%%%%%%%%%%%%%%%%%%%%%%%%%%%%%
\begin{proof}
    We prove the convexity of $\mathcal{AT}$ by showing the union of $\bm{B}_i\bm{u}$ with $\psi_c\in\mathcal{P}$ is convex.
    According to (\ref{eq: pushing dynamics}) and (\ref{eq: linearize dynamics}), such $\bm{B}_i\bm{u}$ can be expressed as the Cartesian product $\mathcal{S}_1\times\mathcal{S}_2$, where $(x,y,\theta)\in\mathcal{S}_1$ and $\psi_c\in\mathcal{S}_2$.
    From (\ref{eq: terminal state approximation}), we know $\mathcal{S}_2$ is an affine transformation of $\mathcal{U}$ and hence is convex.
    In addition, we denote $\left[\bm{R}\bm{A}\bm{J}_c^{\top};0\right]$ as $\left[\bm{C}_1;\bm{C}_2(\psi_c)\right]$, where $\bm{C}_1 \in \mathbb{R}^{2\times4}$ is constant matrix and $\bm{C}_2 \in \mathbb{R}^{1\times4}$ is a continuous function of $\psi_c$.
    Assume $l\in[0,1],\bm{x}_1,\bm{x}_2\in\mathcal{S}_1,\bm{x}(l)=l\bm{x}_1+(1-l)\bm{x}_2$, then $\exists \psi_{c,1},\psi_{c,2}\in\mathcal{P},\bm{u}_1,\bm{u}_2\in\mathcal{U}$, such that $\bm{x}_1=\left[\bm{C}_1;\bm{C}_2(\psi_{c,1})\right]\bm{u}_1,\bm{x}_2=\left[\bm{C}_1;\bm{C}_2(\psi_{c,2})\right]\bm{u}_2$, and their convex combination $\bm{x}(l)=\left[\bm{C}_1\bm{u}(l);l\bm{C}_2(\psi_{c,1})\bm{u}_1+(1-l)\bm{C}_2(\psi_{c,2})\bm{u}_2\right]$.
    Since $\bm{C}_2$ is continuous, $\bm{C}_2\bm{u}$ is monotonic, $\exists\psi_c^{*}\in\mathcal{P}$, such that $l\bm{C}_2(\psi_{c,1})\bm{u}_1+(1-l)\bm{C}_2(\psi_{c,2})\bm{u}_2=\bm{C}_2(\psi_c^{*})\bm{u}(l)$, where $\bm{u}(l)=l\bm{u}_1+(1-l)\bm{u}_2\in\mathcal{U}$.
    So far, the convexity of $\mathcal{S}_1$ and $\mathcal{AT}$ are proved.
\end{proof}

Based on (\ref{eq: augmented terminal state approximation}), we can obtain the analytic approximation of (\ref{eq: reachable set definition}) as:
\begin{equation}
    \mathcal{R}_{\tau}(\bar{\bm{x}})\approx
    \bigcup_{i=1,\dots,N
                \atop
             j\in\{st,sl,sr\}}
             \mathcal{AT}_{\tau}(\bar{\bm{x}})_{ij}
    \label{eq: reachable set approximation},
\end{equation}
an example of such a set is shown in Fig. ~\ref{fig: CA3P and reachable set} \textbf{Right}. 
% Note that $\bm{f}_i(\bar{\bm{x}},\bar{\bm{u}})=\bm{B}_i(\bar{\bm{x}}){\bar{\bm{u}}}$ holds for the pushing system (\ref{eq: pushing dynamics}), thus $\bar{\bm{u}}$ no longer appears in 
% (\ref{eq: terminal state approximation})$\sim$
% (\ref{eq: reachable set approximation}).
The reachable set contains new slider poses to be reached from a certain pose with greater probability and naturally enables a discrete choice of the contact face and mode. Since it is still the union of multiple convex sets, computationally efficient methods can be derived for nearest neighbor search \cite{Wu2020R3TRR}.
% Techniques for state connection will be stated in Sec.~\ref{sec: planar pushing planning}.
\vspace{-5pt}
\subsection{Object Interaction Model} \label{sec: object interaction model}
\begin{figure}[t]
    \centering
    \includegraphics[width=.95\columnwidth]{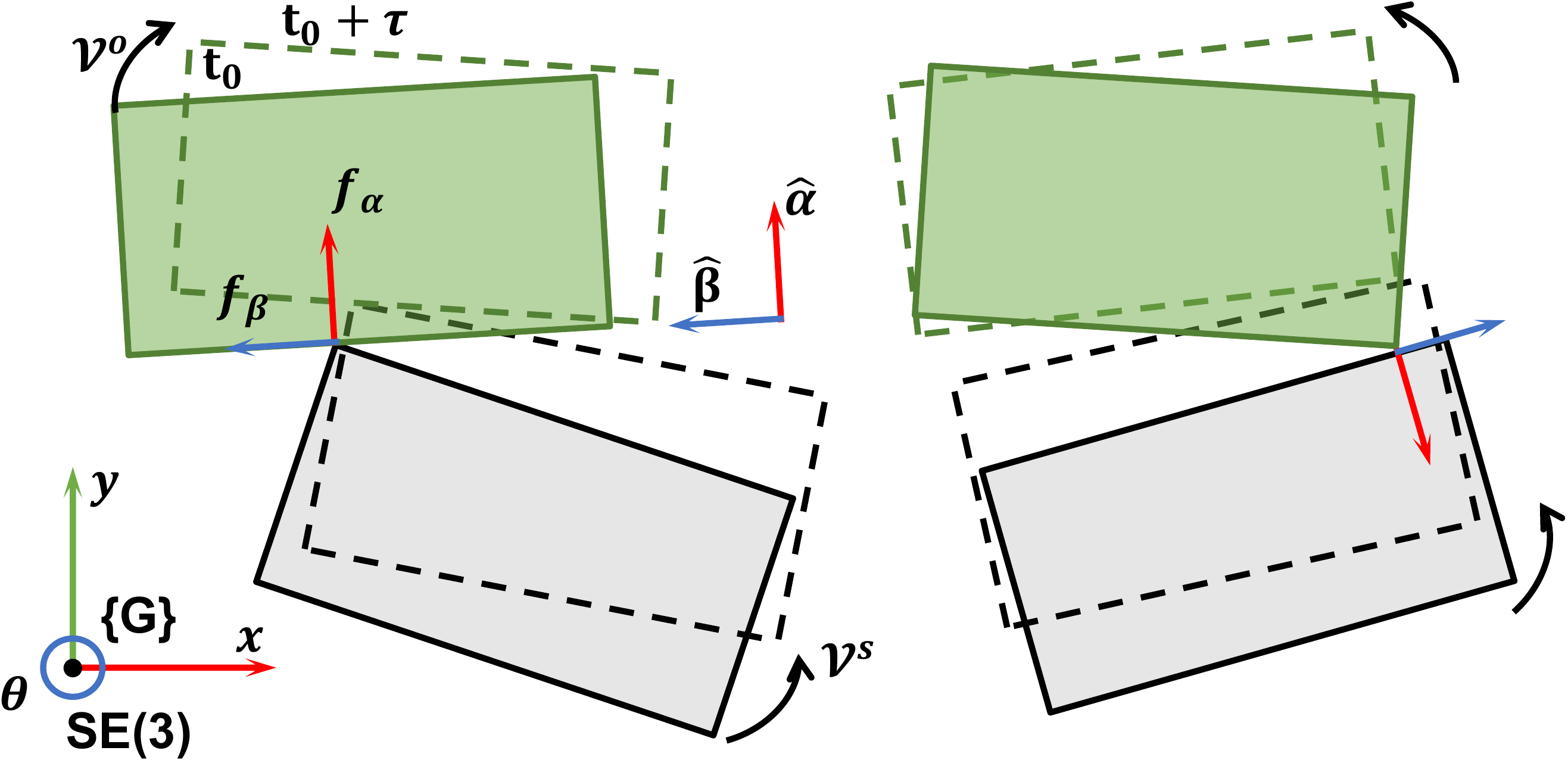}
    \caption{
        Two different configurations of object interaction. $\mathcal{V}^s,\mathcal{V}^s$ are instantaneous twists, $\hat{\bm{\alpha}},\hat{\bm{\beta}}$ are the contact normal and tangential, $f_{\alpha},f_{\beta}=f_{\beta+}-f_{\beta-}$ are contact forces.
        \textbf{Left}: $\hat{\bm{\alpha}}$ points inwards the obstacle (painted in green).
        \textbf{Right}: $\hat{\bm{\alpha}}$ points inwards the slider (painted in grey).
    }
    \label{fig: interaction model definition}
    \vspace{-18pt}
\end{figure}
The object interaction model is designed to forecast the motion of movable objects in contact with the planar slider. Fixed objects in Sec.~\ref{sec: retrieval problem formulation} refer to objects that fall over easily under unintentional contacts. It is dangerous to enable interaction with these objects in planning, even with the physics engine, because control error may lead to large divergence. Besides, relying on a physics engine for precise motion prediction is time-consuming and unnecessary.
We denote the contact normal and tangential as $\hat{\bm \alpha}$ and $\hat{\bm \beta}$. Without loss of generality, we assume $\hat{\bm \alpha}, \hat{\bm \beta}$ form a right-handed coordinate system, and $\hat{\bm \alpha}$ points inwards the obstacle, as seen in Fig. \ref{fig: interaction model definition} \textbf{Left}. The contact force is defined as $\bm{f}=\left[f_{\alpha},f_{\beta+},f_{\beta-}\right]^{\top} \succeq 0$. Given the twist of slider and obstacle in global coordinates $\mathcal{V}^s,\mathcal{V}^o$, and the contact jacobians $\bm{J}_c^s,\bm{J}_c^o$, the non-penetration constraints can be expressed as:
\begin{equation}
    \begin{aligned}
        0 \leq \hat{\bm \alpha}^{\top}
        \left(
            \bm{J}_c^o\mathcal{V}^o-\bm{J}_c^s\mathcal{V}^s
        \right) 
        & \perp 
        f_{\alpha} \geq 0 \\
        \bm{0} \preceq
        \begin{bmatrix}
        \hat{\bm \beta}^{\top} \\
        -\hat{\bm \beta}^{\top}
        \end{bmatrix}
        \left(
            \bm{J}_c^o\mathcal{V}^o-\bm{J}_c^s\mathcal{V}^s
        \right) + \lambda
        \begin{bmatrix}
        1 \\
        1
        \end{bmatrix}
        & \perp 
        \begin{bmatrix}
        f_{\beta+} \\
        f_{\beta-}
        \end{bmatrix} \succeq \bm{0} \\
        0 \leq
        \mu f_{\alpha}-f_{\beta+}-f_{\beta-}
        & \perp
        \lambda
        \geq 0
    \end{aligned}
    \label{eq: non-penetration constraints}.
\end{equation}
% The matrices $\hat{\bm \alpha}, [\hat{\bm \beta}, -\hat{\bm \beta}]$ are rewritten as $\bm{M}_{\alpha},\bm{M}_{\beta}$ for brevity. 

Hence, the feasible contact force 
% $\bm{f}=\left[f_{\alpha},f_{\beta+},f_{\beta-}\right]^{\top}$
that satisfies the constraints in (\ref{eq: non-penetration constraints}) is equivalent to the solution of the Linear Complementarity Problem (LCP):
\begin{equation}
    % \begin{bmatrix}
    %     \bm{M}_{\alpha}^{\top}\bm{K}^o\bm{M}_{\alpha} & \bm{M}_{\alpha}^{\top}\bm{K}^o\bm{M}_{\beta} & 0 \\
    %     \bm{M}_{\beta}^{\top}\bm{K}^o\bm{M}_{\alpha} & \bm{M}_{\beta}^{\top}\bm{K}^o\bm{M}_{\beta} & \bm{1} \\
    %     \mu & -\bm{1}^{\top} & 0
    % \end{bmatrix}
    % \begin{bmatrix}
    %     f_{\alpha} \\
    %     f_{\beta+} \\
    %     f_{\beta-} \\
    %     \lambda
    % \end{bmatrix}+
    % \begin{bmatrix}
    %     -{\bm M}_{\alpha}^{\top}{\bm J}_c^s\mathcal{V}^s \\
    %     -{\bm M}_{\beta}^{\top}{\bm J}_c^s\mathcal{V}^s \\
    %     0
    % \end{bmatrix}
    \begin{aligned}
        \bm{z} & = 
        \begin{bmatrix}
            \bm{M}^{\top}\bm{K}^o\bm{M} & 
            \begin{array}{l}
                0  \\
                \bm{1}
            \end{array} \\
            \begin{array}{cc}
                \mu & -\bm{1}^{\top}
            \end{array} &
            0
        \end{bmatrix}
        \begin{bmatrix}
            \bm{f} \\
            \lambda
        \end{bmatrix} + 
        \begin{bmatrix}
            -\bm{M}^{\top}\bm{J}_c^s\mathcal{V}^s \\
            0
        \end{bmatrix} \\
        \bm{0} & \preceq \bm{z} \perp 
        \begin{bmatrix}
            \bm{f} \\
            \lambda
        \end{bmatrix}
        \succeq \bm{0}
    \end{aligned}
    \label{eq: LCP problem},
\end{equation}
where $\bm{z}$ is the auxiliary variable, $\bm{M}=[\hat{\bm \alpha},\hat{\bm \beta},-\hat{\bm \beta}]$, and $\bm{K}^o=\bm{J}_c^o\bm{A}{\bm{J}_c^o}^{\top}$ maps the contact force to contact point velocity. The LCP (\ref{eq: LCP problem}) can be efficiently solved with Newton-based methods \cite{Fischer1995ANM}. Then, new poses of the movable obstacle 
$o_k^m \in \mathcal{O}^m$
can be obtained through forward integration:
\begin{equation}
    \bm{x}_k[t_0+\tau] = \bm{x}_k[t_0] + \tau\bm{R}(\theta_k[t_0])\bm{A}{\bm{J}_c^o}^{\top}\bm{M}\bm{f}
    \label{eq: LCP forward integration},
\end{equation}
where
% $k\in \left\{1,\dots,\vert\mathcal{O}^m\vert\right\}$
$\bm{R}(\theta_k[t_0]) \in \mathbb{R}^{3 \times 3}$ is the rotation matrix.
Similar conclusions as (\ref{eq: non-penetration constraints})$\sim$(\ref{eq: LCP forward integration}) can be drawn for a different scenario in Fig. \ref{fig: interaction model definition} \textbf{Right}, where $\hat{\bm \alpha}$ points inwards the slider. 

Based on the abovementioned techniques, the lightweight interaction model (\ref{eq: non-penetration constraints})$\sim$(\ref{eq: LCP forward integration}) balances safety and efficiency in push planning.
\vspace{-5pt}
\subsection{Model Predictive Pushing Controller} \label{sec: model predictive pushing controller}
To execute trajectories planned by CA3P, we adopted the Math Programming with Complementarity Constraints (MPCC) formulation in \cite{Moura2021NonprehensilePM} with compensation of external disturbance, which refers to the additional contact force during interaction with movable obstacles.
The disturbance is modeled as a time-varying term in the nominal dynamics
\vspace{-3pt}
\begin{equation}
    \vspace{-3pt}
    \dot{\bm{x}}(t)=\bm{f}(\bm{x}(t),\bm{u}(t))+\hat{\bm{d}}(t)
    \label{eq: time-varying disturbance}.
\end{equation}
The dynamic constraints in MPCC are compensated by the estimated disturbance
\vspace{-3pt}
\begin{equation}
    \vspace{-3pt}
    \bm{x}[k+1]=\bm{x}[k]+\tau_{\text{MPC}}\left[\bm{f}(\bm{x}[k],\bm{u}[k])+\hat{\bm{d}}[k]\right]
    \label{eq: MPCC dynamic constraints disturbance compensation},
\end{equation}
where $\tau_{\text{MPC}}$ is the control time step. The estimated disturbance is updated as
\vspace{-3pt}
\begin{equation}
    \vspace{-3pt}
    \dot{\hat{\bm{d}}}[k]=\kappa_d(\bm{x}_{\text{obs}}[k]-\bm{x}_{\text{mpc}}[k])
    \label{eq: disturbance estimation update law},
\end{equation}
where $\bm{x}_{\text{obs}}[k],\bm{x}_{\text{mpc}}[k]$ are observed and predicted states, respectively, $\kappa_d > 0$ denotes the update rate. The compensation for contact force reinforces closed-loop feedback and avoids modeling the multi-slider control problem with complex nonlinear MPC.
\begin{figure*}[!ht]
    \centering 
        \subfigure[scene 0]{
            \includegraphics[width=1.2in]{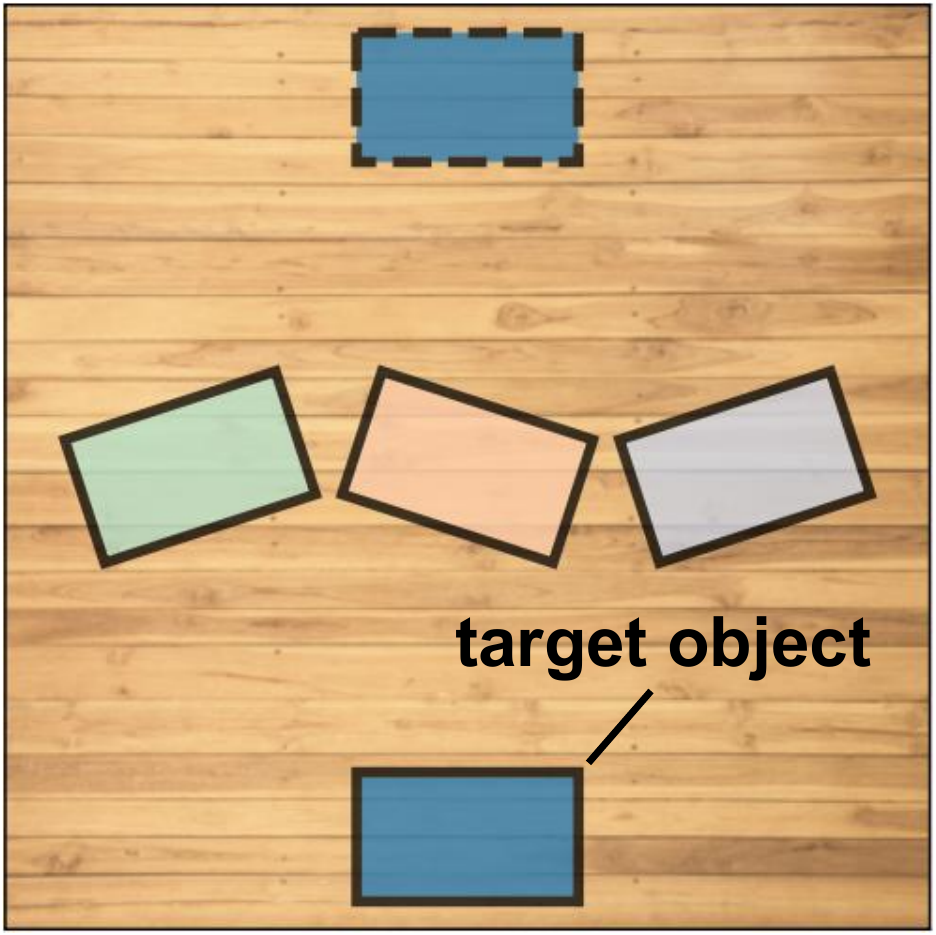}
            \label{fig: exp scene0 visual}
        }
        \subfigure[scene 1]{
            \includegraphics[width=1.2in]{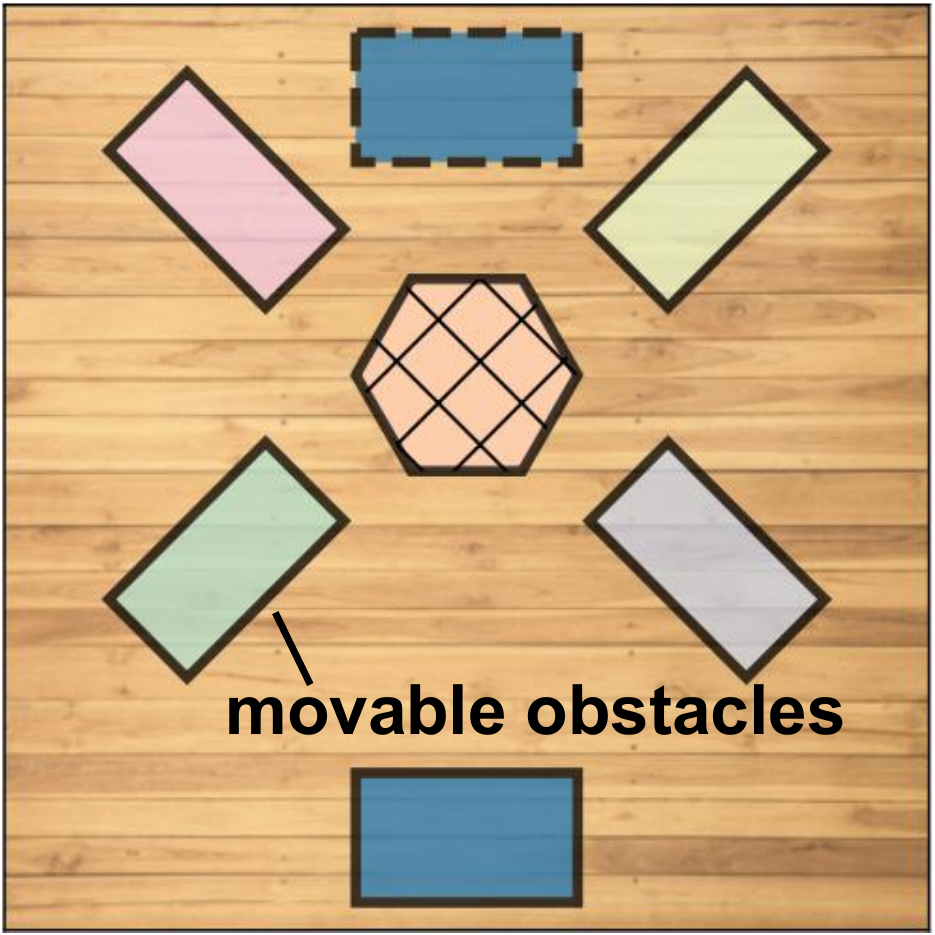}
            \label{fig: exp scene1 visual}
        }
        \subfigure[scene 2]{
            \includegraphics[width=1.2in]{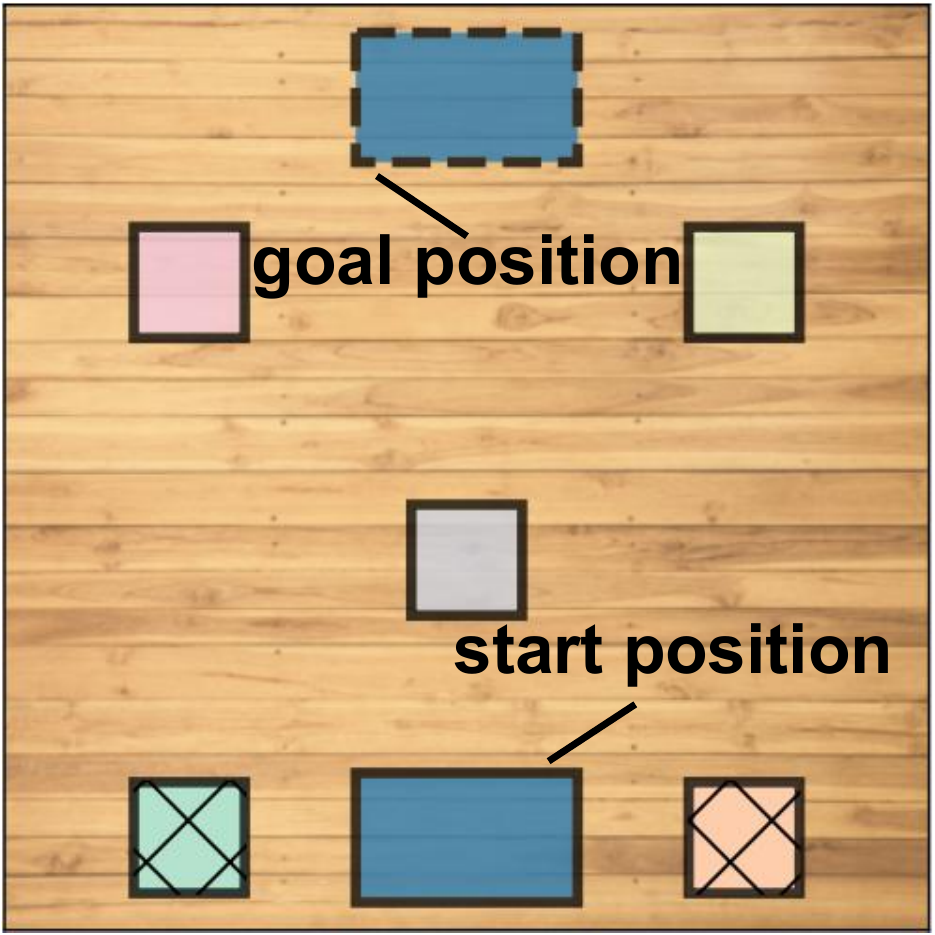} 
            \label{fig:exp scene2 visual}
        }
        \subfigure[scene 3]{
            \includegraphics[width=1.2in]{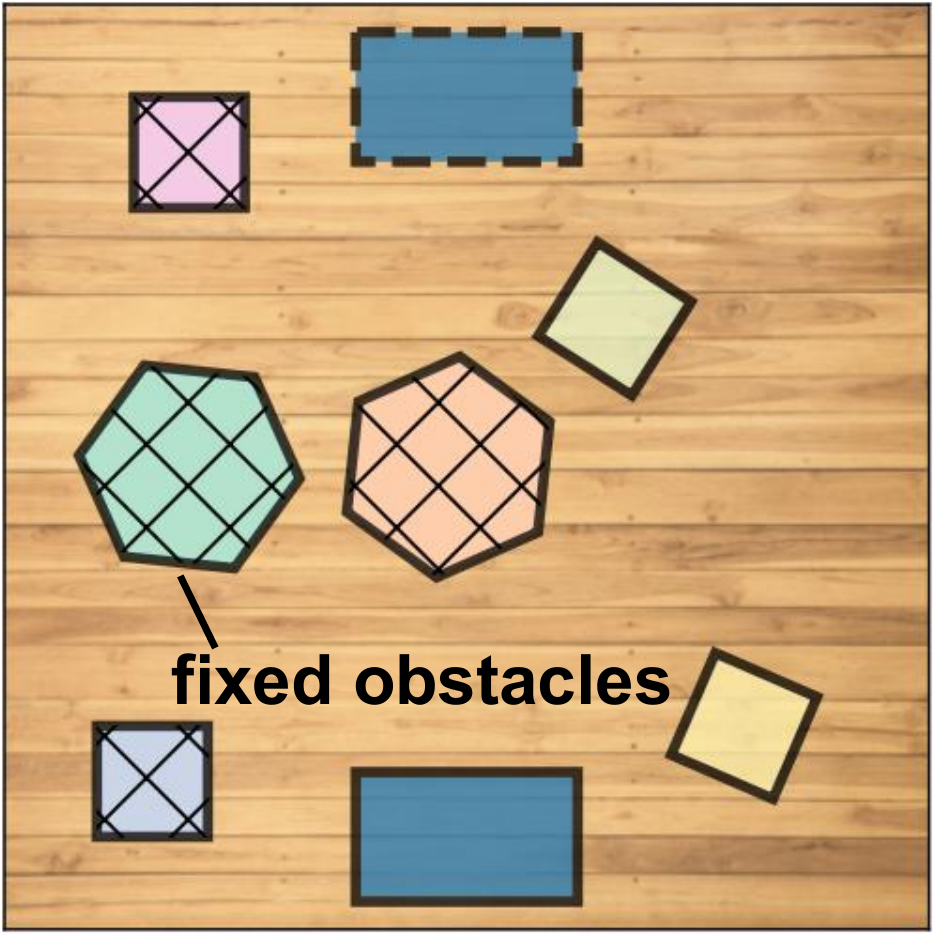}
            \label{fig: exp scene3 visual}
        }
        \subfigure[scene 4]{
            \includegraphics[width=1.2in]{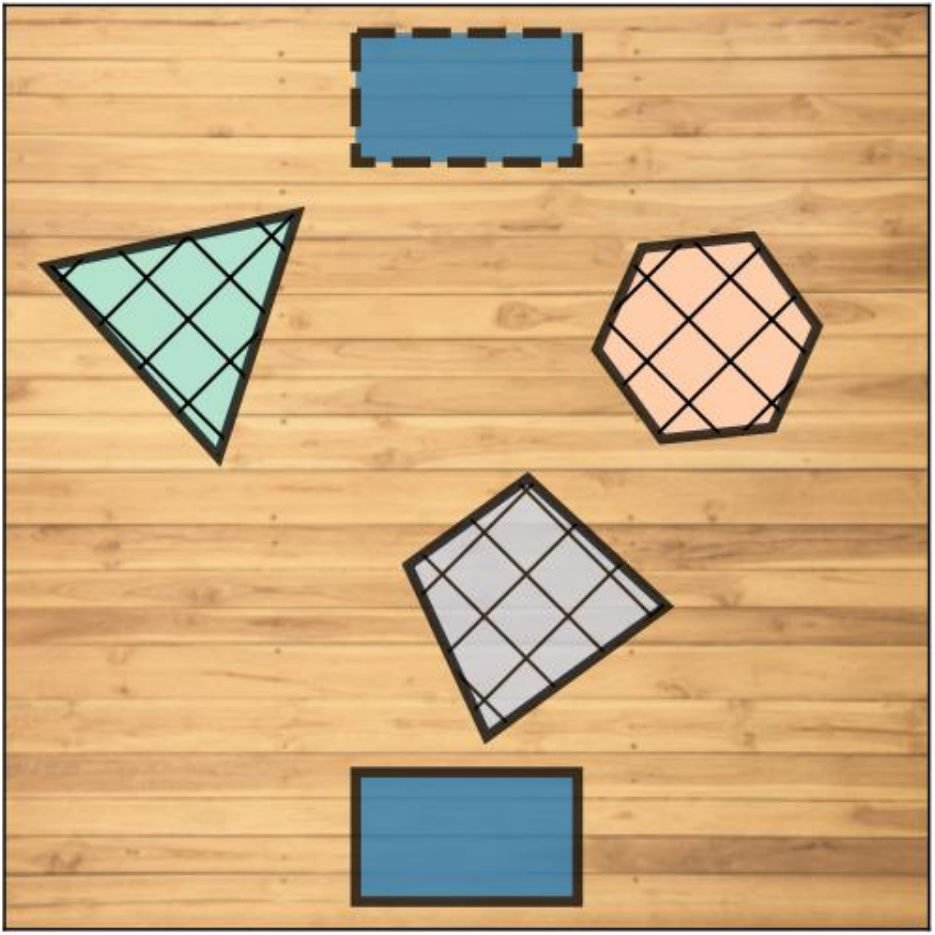}
            \label{fig: exp scene4 visual}
        }
    \vspace{-8pt}
    \caption{Five representative problem instances. Start and goal position of the \textbf{\textcolor[RGB]{31,119,180}{target}} objects are painted with solid and dashed borders, respectively. Fixed obstacles are marked with net mesh.
    The tasks in \textbf{scene 0} and \textbf{scene 1} require the obstacles to be cleared away. The search process for \textbf{scene 2} and \textbf{scene 3} can be accelerated by pushing aside the obstacles, although not essential. \textbf{Scene 4} only needs obstacle avoidance.
    }
    \label{fig: exp five planning scenes}
    \vspace{-15pt}
\end{figure*}
\begin{figure}[t]
\centering
\subfigure{\includegraphics[height=4.75cm]{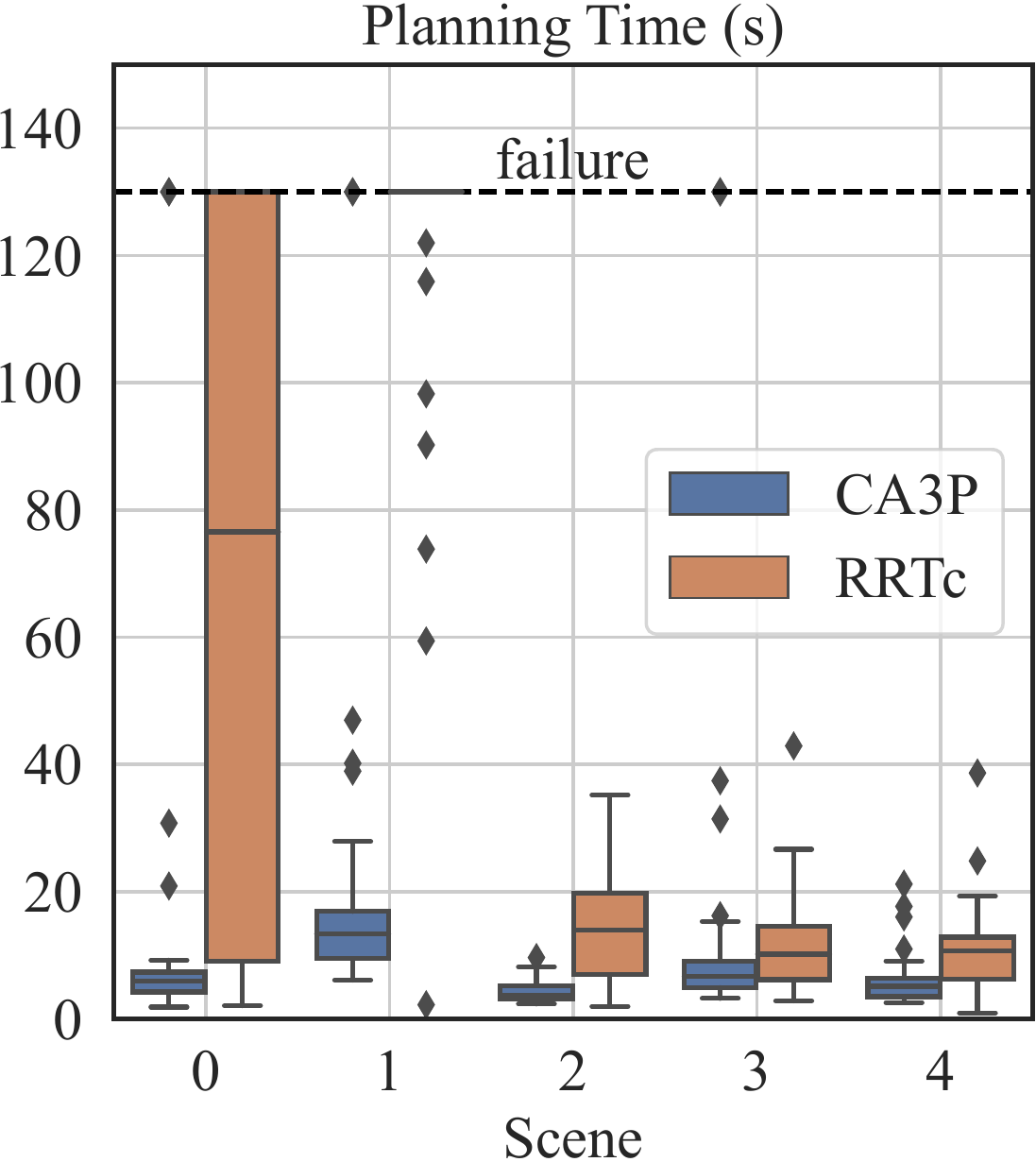}}
\subfigure{\includegraphics[height=4.75cm]{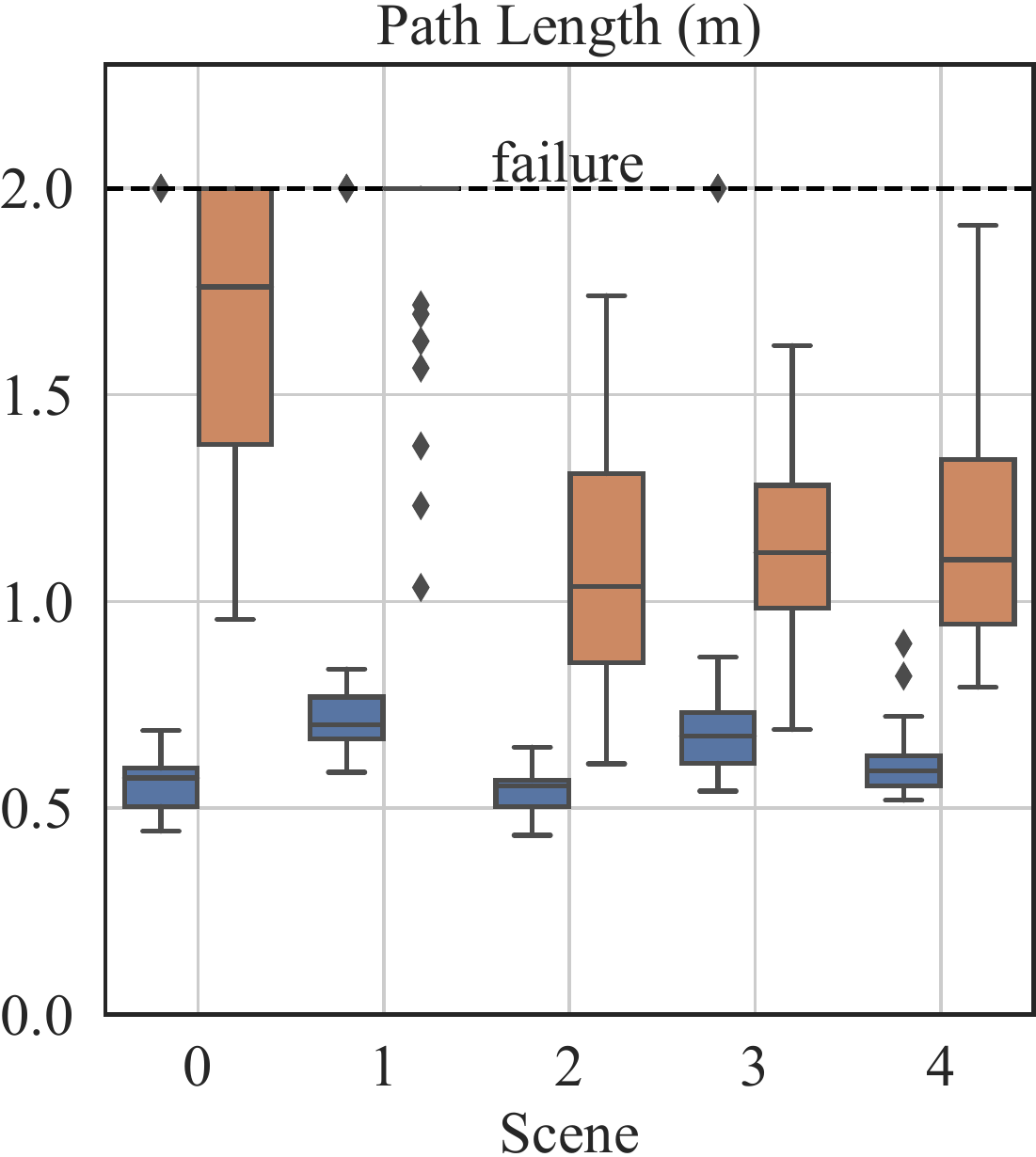}}
\vspace{-18pt}
\caption{Results of CA3P and RRTc on five instances across 30 trials. The planning time and path length are set as \SI{130}{s} and \SI{2.0}{m} for failure trials. The median, lower and upper quartiles, error bars, and outlier values are reported. Less planning time means the method is more efficient. A shorter path length means the result is more optimal.}
\vspace{-23pt}
\label{fig: exp planning time path length}
\end{figure}
%

% \section{DYNAMIC REINFORCEMENT LEARNING}
% \input{chapters/4_method}

\section{RESULTS}
In this section, we show that CA3P generates a shorter path in less time compared with baselines, utilizing the contact-aware feature. Moreover, we demonstrate through robot experiments that the proposed algorithm is capable of executing the planned path with an acceptable error, even if the slider is obstructed by obstacles.
%

% \vspace{-3pt}
% \subsection{Experimental setup} \label{sec: exp-setup}
% \vspace{-3pt}
% %
% - to add -
% %

\vspace{-3pt}
\subsection{Simulation Studies} \label{sec: exp-push planning}
\vspace{-3pt}
The simulations were conducted on a 64-bit Intel Core i7-12700 4.9GHz Ubuntu workstation with 32GB RAM. We used Shapely \cite{Gillies_Shapely_2023} for collision detection and visualization.
We compared CA3P with the following baselines:
\subsubsection{RRTc}
An RRT-based planner utilizing the differential flatness properties proposed in \cite{Zhou2019PushingR}. The slider is manipulated through sticking contacts and forced to follow trajectories with constant curvature, i.e., Dubins path. Whenever a new sample and its nearest neighbor are generated, the corresponding contact point (including the contact face) and force direction are obtained through differential flat mapping. We hypothesize that this method imposes tight constraints on motion planning and could hardly find a nearly optimal path.
%
% \vspace{6pt}
\subsubsection{MPCC}
The optimization-based scheme proposed in \cite{Moura2021NonprehensilePM}, which uses nonlinear programming with obstacle avoidance constraints. Since this method is difficult to solve with the number of obstacles we considered, we approximated the obstacles by their maximum inscribed circles to make more space. We hypothesize that this method is less flexible when handling obstacles.
% \usepackage{multirow}
% \usepackage{multirow}

% \usepackage{multirow}
\begin{table}[]
\caption{Success rate and number of nodes}
\label{tab: exp success rate and node num}
\vspace{-5pt}
\resizebox{\linewidth}{!}{
\begin{tabular}{ccccccc}
% \hline
\toprule[1.2pt]
\multicolumn{2}{c}{\textbf{Scene}}            & \textbf{0}       & \textbf{1}       & \textbf{2}      & \textbf{3}       & \textbf{4}       \\ 
% \hline
\midrule[0.8pt]
\multirow{2}{*}{\textbf{Success}}      & CA3P & 28/30   & 28/30   & 30/30  & 28/30   & 30/30   \\ \cline{2-2}
                              & RRTc & 19/30   & 7/30    & 30/30  & 30/30   & 30/30   \\ \cline{1-2}
\multirow{2}{*}{\textbf{$\bullet$ \rm{\textbf{Node In Tree}}}} & CA3P & 263±219 & 544±347 & 169±67 & 286±242 & 210±135 \\ \cline{2-2}
                              & RRTc & 107±117 & 302±160 & 25±14  & 22±12   & 31±21   \\ 
% \hline
\bottomrule[1.2pt]
\multicolumn{4}{l}{\small $\bullet$ reported as mean±standard deviation.}\\
\end{tabular}}
\vspace{-20pt}
\end{table}

\vspace{4pt}
For CA3P, we set the goal sampling bias as $\SI{0.1}{}$, and set $\bar{f}=\SI{0.15}{N},\bar{\dot{\psi}}_c=\SI{1.0}{rad/s},\mu_p=\SI{0.2}{},\tau=\SI{0.05}{s},\tau_{\text{LQR}}=\SI{0.01}{s}$. For all methods, we define the stopping criterion as maximum planning time $\SI{1e3}{s}$ or maximum number of nodes $\SI{1e3}{}$. A method will return failure if the goal region is unreachable from explored states when the criterion is met. To evaluate the effectiveness of the proposed method, we generated 5 representative problem instances, as shown in Fig. \ref{fig: exp five planning scenes}. The planning time and total path length for each instance across 30 trials are shown in Fig. \ref{fig: exp planning time path length}, and the success rate and the number of nodes are presented in Table. \ref{tab: exp success rate and node num}.
The proposed CA3P greatly reduced planning time and generated shorter trajectories. Moreover, CA3P reports narrower interquartile ranges for all problem instances, indicating that the method achieves more stable performance with random scenes and trials. For more complicated scene 0 and scene 1, CA3P increased the success rate by $30\%$ and $70\%$. 
% Surprisingly, RRTc could sometimes complete the task utilizing the outer boundary of the search space. 
We observed that RRTc completed all the trials if the search space was enlarged by $40\%$ on each side, at the expense of adding $20\%$ to trajectory length. Nevertheless, such a compromise is impractical due to the constrained environment and limited workspace of the manipulator. Results of scene 2 and scene 3 showed that actively removing obstacles is an effective way to obtain consistent path length; since avoiding obstacles is pretty demanding in the sampling sequence and quality. Results of scene 4 proved that the reachable set offers preferable directions in the search process.
Whatever we did to adjust the parameters and manually decide the contact face, the MPCC baseline failed to solve all the problem instances, either due to collision or the large distance towards the goal. The result is mainly because the controllability of the planar slider is significantly restricted if it is not allowed to switch contact faces.
\begin{figure*}[tbp]
    \centering 
    \subfigcapskip=0pt
        \subfigure[$t=$\SI{1.03}{s}]{
            \includegraphics[height=1.06in]{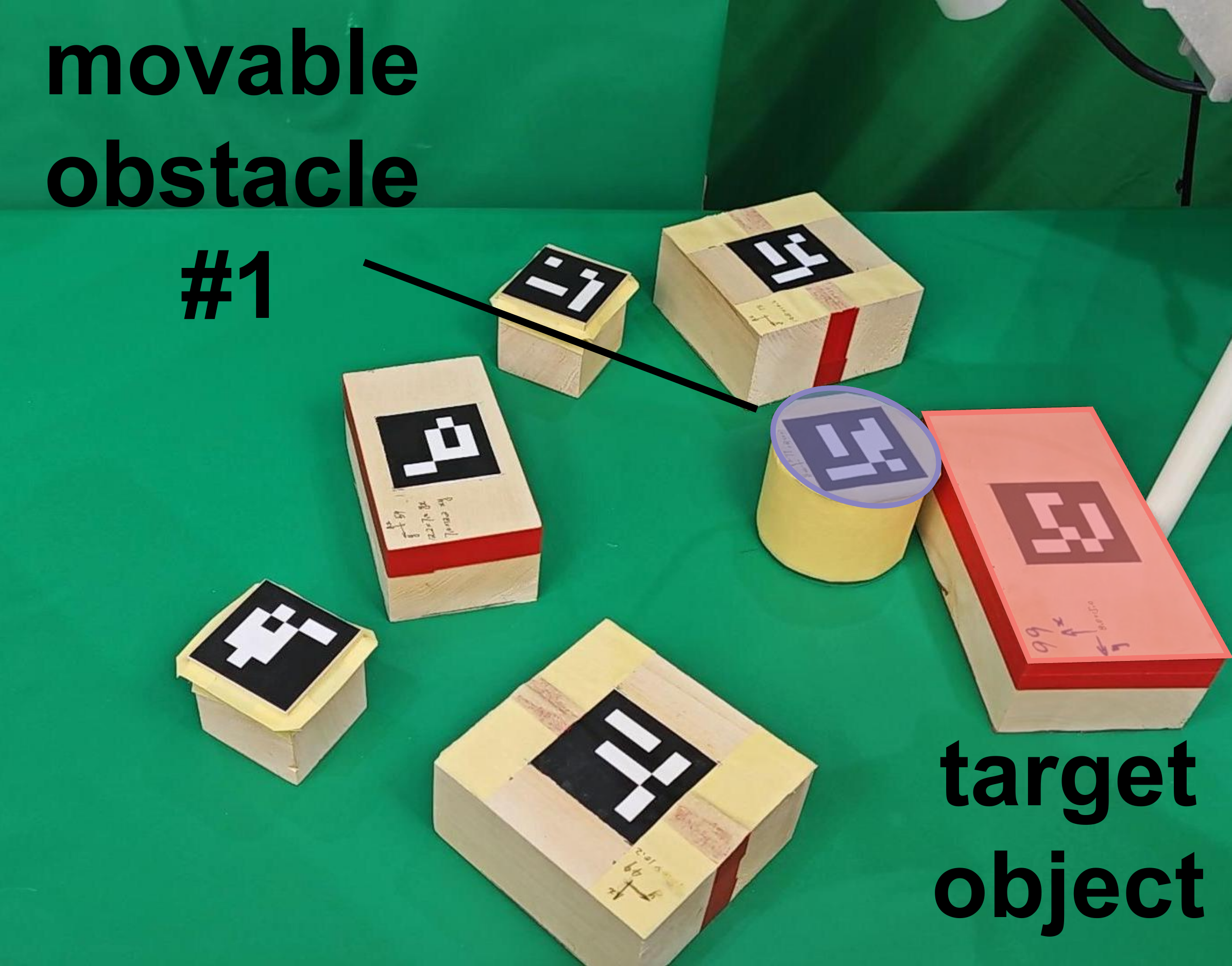}
            \label{fig: exp retrieval frame 0}
        }\hspace{-3.3mm}
        \subfigure[$t=$\SI{12.84}{s}]{
            \includegraphics[height=1.06in]{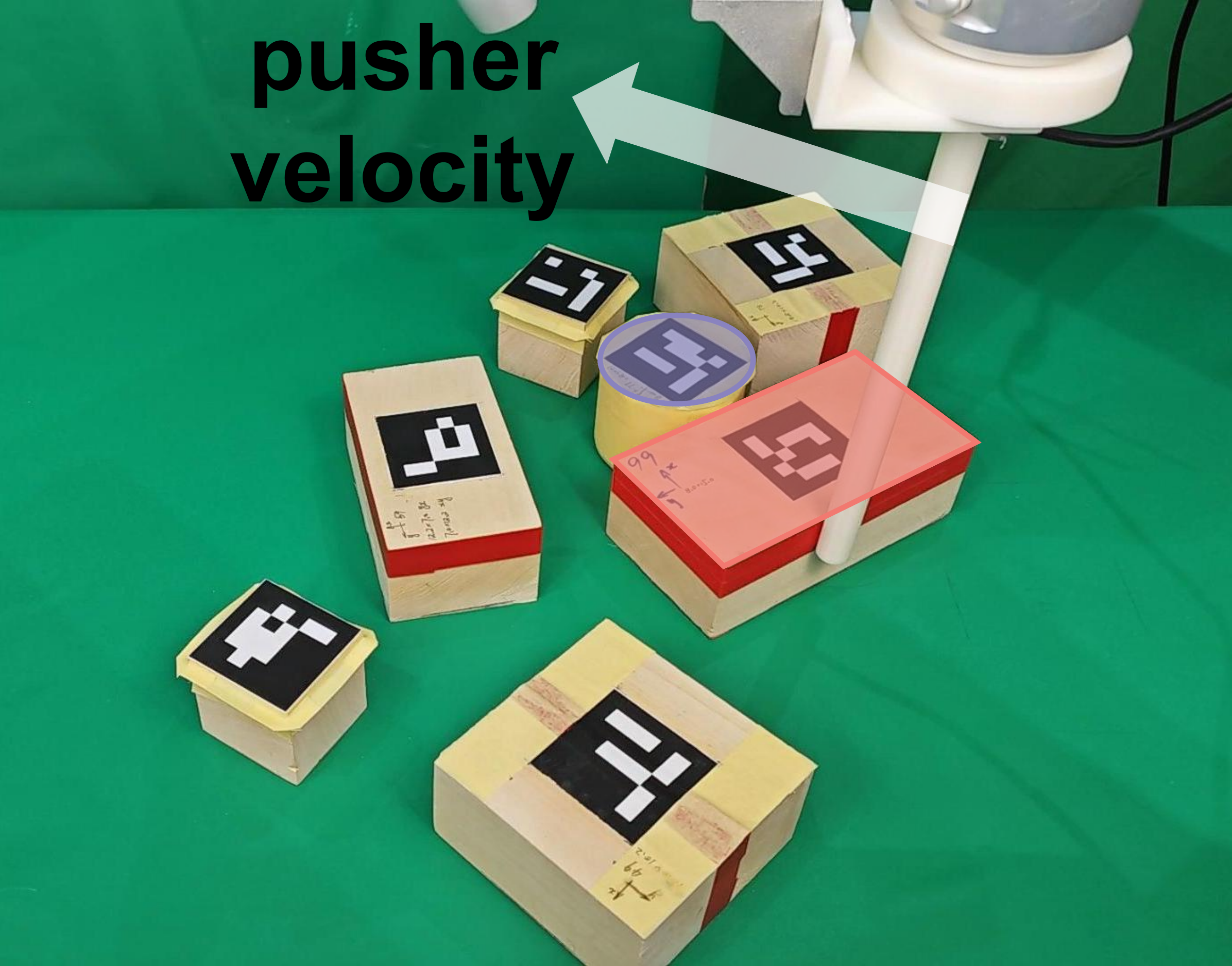}
            \label{fig: exp retrieval frame 1}
        }\hspace{-3.3mm}
        \subfigure[$t=$\SI{22.49}{s}]{
            \includegraphics[height=1.06in]{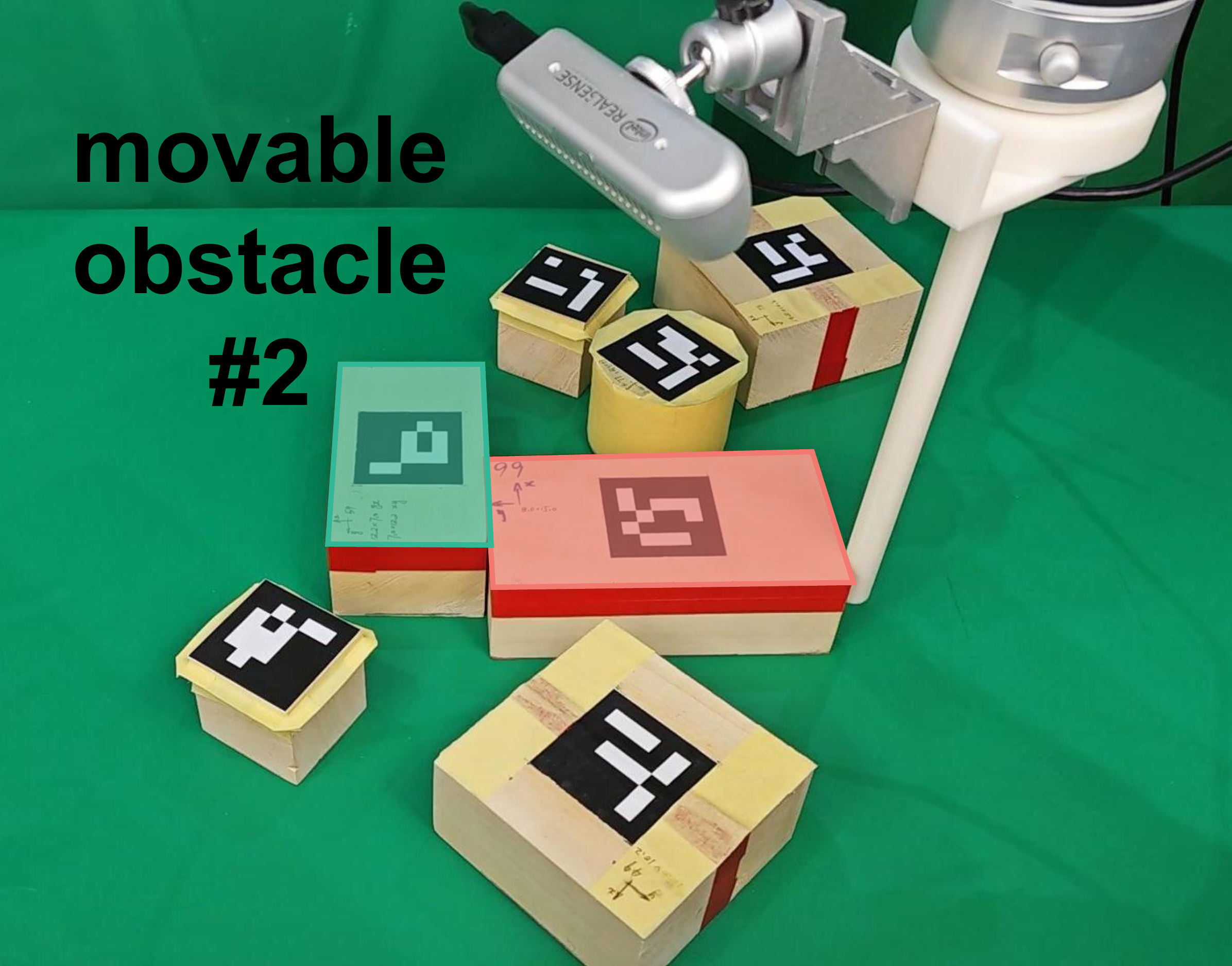} 
            \label{fig: exp retrieval frame 2}
        }\hspace{-3.3mm}
        \subfigure[$t=$\SI{39.36}{s}]{
            \includegraphics[height=1.06in]{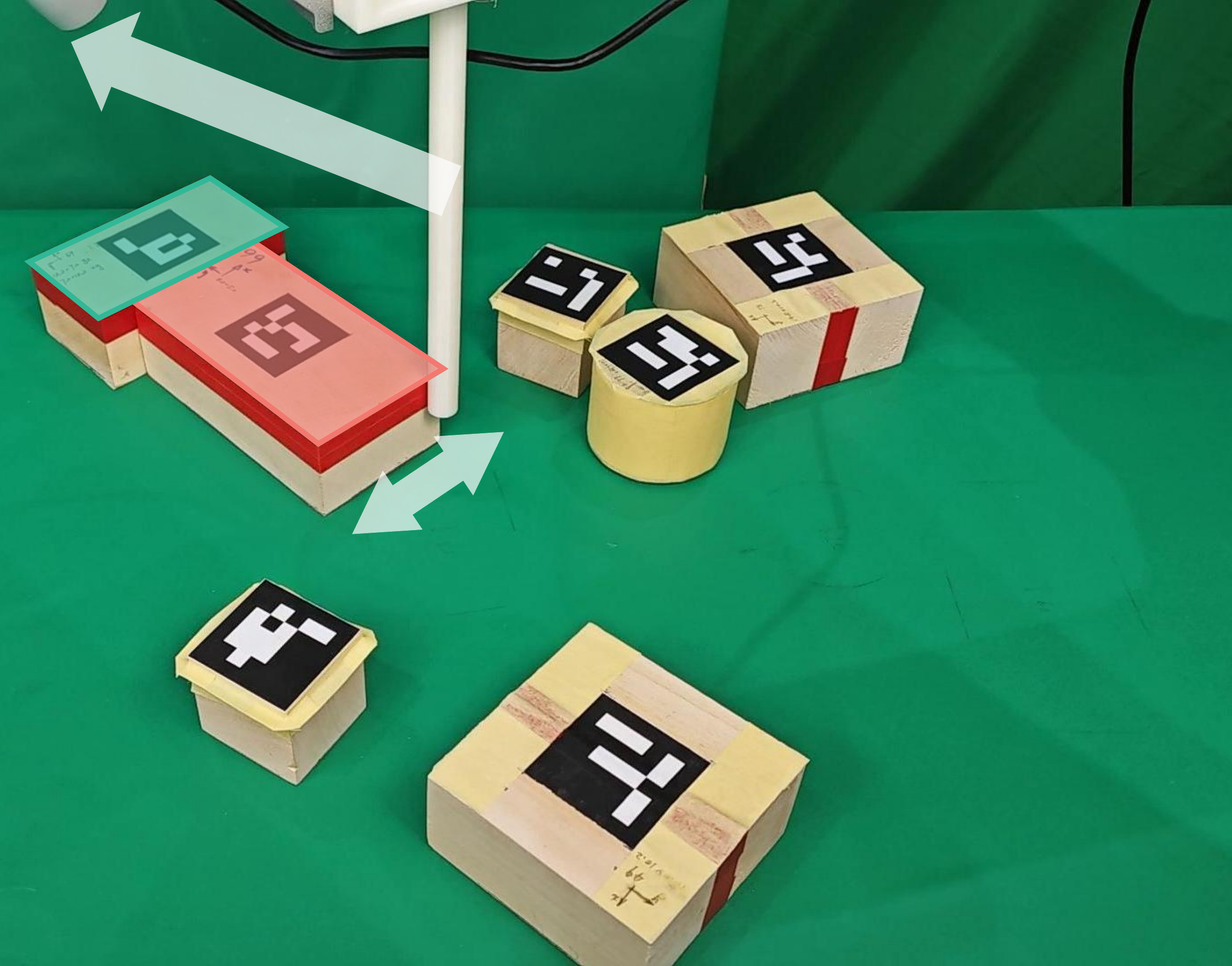}
            \label{fig: exp retrieval frame 3}
        }\hspace{-3.3mm}
        \subfigure[$t=$\SI{50.57}{s}]{
            \includegraphics[height=1.06in]{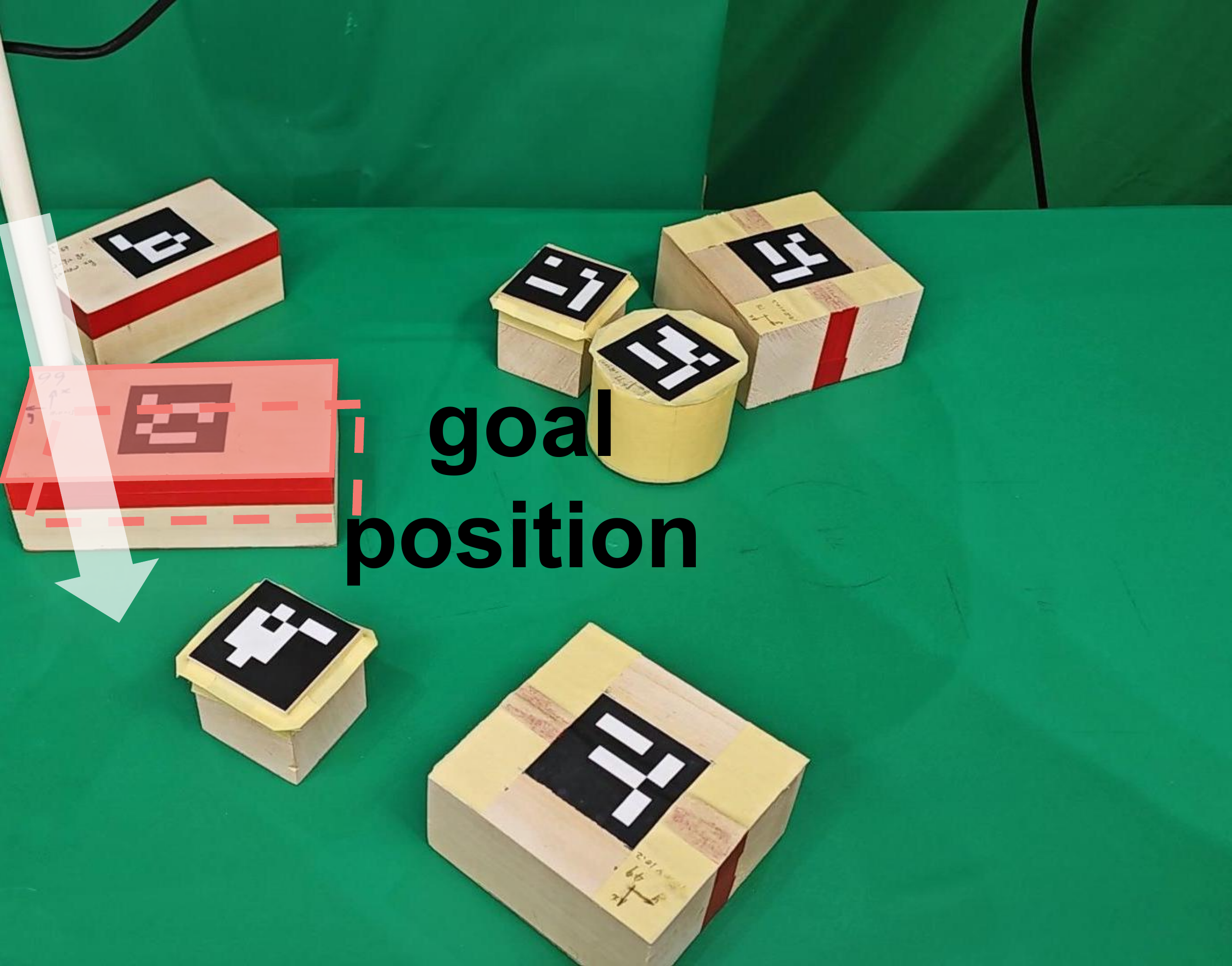}
            \label{fig: exp retrieval frame 4}
        }
    \vspace{-5pt}
    \caption{
        Snapshots of the planar object retrieval task executed on a UR5 robot arm. The slider was manipulated to consecutively push aside a cylindrical (a-b) and a cubic obstacle (c-d) and was finally pushed to the goal position after switching contact faces. The time consumed on switching faces is ignored.
    }
    \label{fig: exp complete retrieval task}
    \vspace{-11pt}
\end{figure*}
\begin{figure}[tbp]
    \centering
    \vspace{0pt}
    \subfigtopskip=0pt
    \subfigcapskip=-5pt
    \subfigure[]{
        \includegraphics[height=3.83cm]{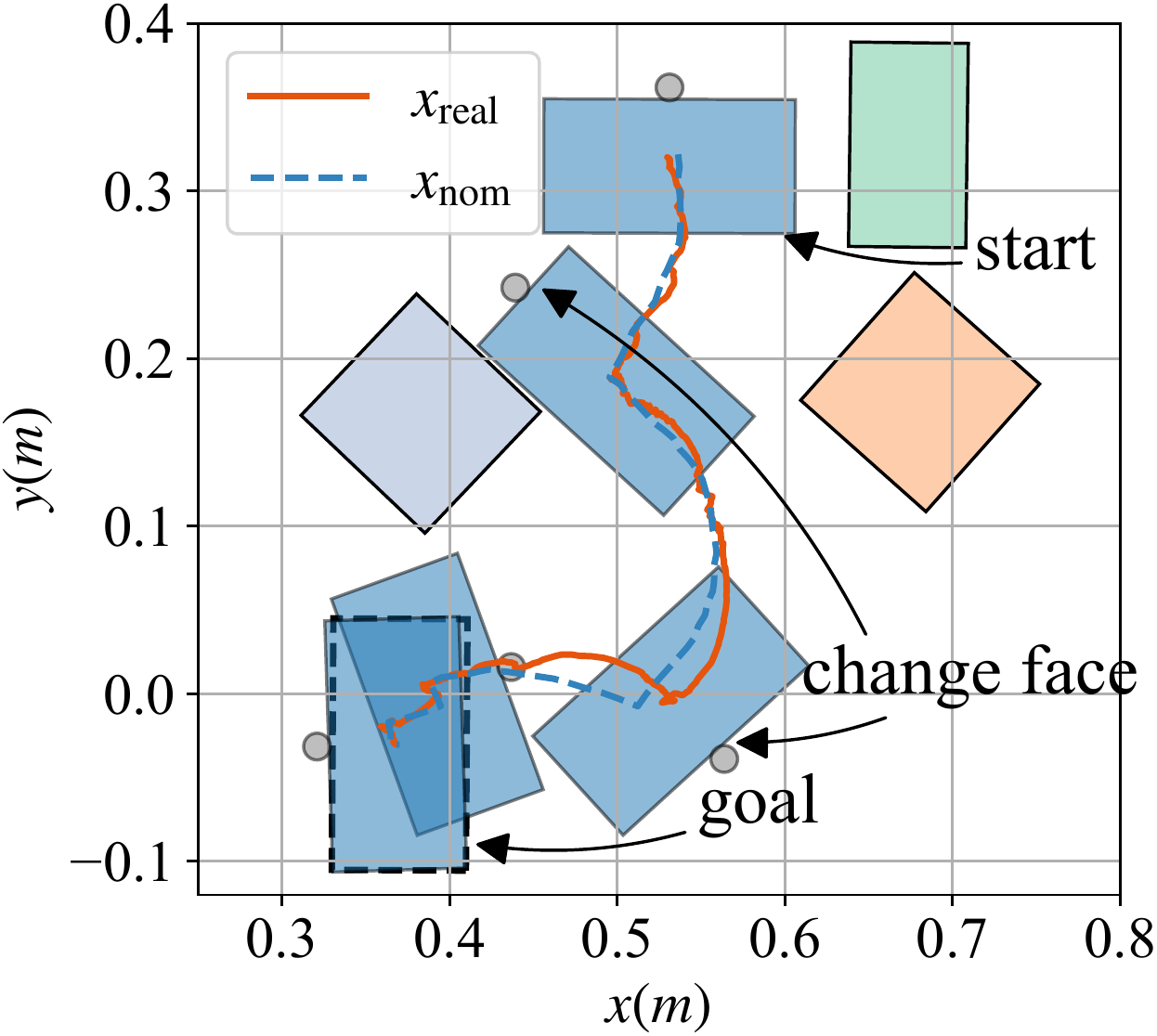}
        \label{fig: exp avoidance trajectory}
    }
    \hspace{-17.0pt}
    \subfigure[]{
        \includegraphics[height=3.83cm]{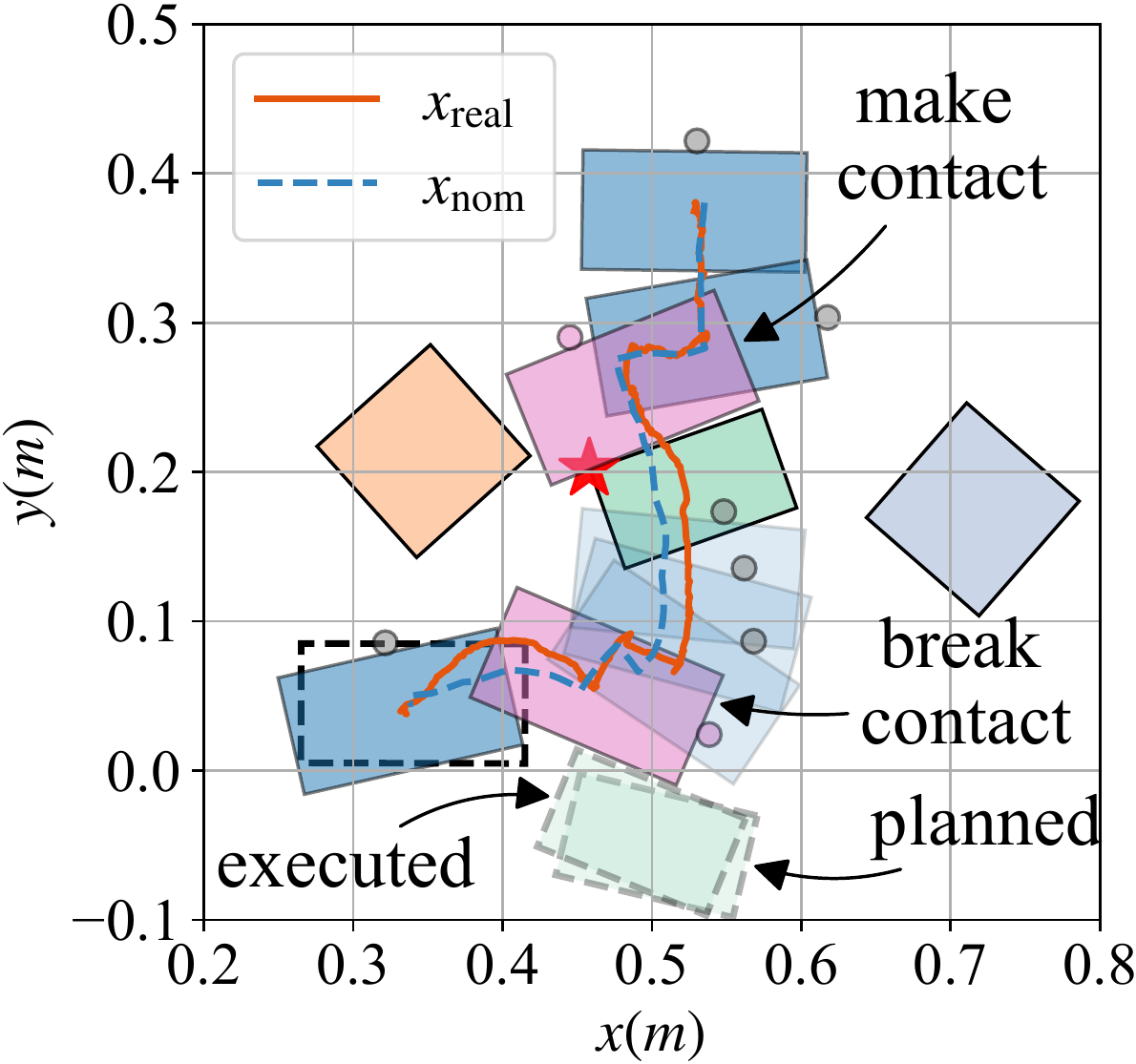}
        \label{fig: exp pushaway traj}
    }
    \vspace{-6pt}
    \caption{Planned (nominal) and executed trajectory of the planar slider. Several keyframes before switching faces are depicted with the round pusher.
    (a) Obstacle avoidance. (b) Removal of the \textbf{\textcolor[RGB]{127,201,127}{green}} object through pushing.
    }
    \label{fig: exp track and pushaway traj}
    \vspace{-16pt}
\end{figure}
%
% \subsection{Numerical Studies of Planar Object Retrieval} \label{sec: exp-numerial simulation}
% %
% - to add -
% %

\subsection{Real-World Experiments}  \label{sec: exp-robot experiment}
We implemented the robot experiments on a 64-bit Intel Core i7-9700 4.7GHz Ubuntu workstation with 16GB RAM. We mounted a $\Phi\SI{15}{}\times\SI{250}{mm}$ resin pusher on a UR5 robot. The perception system was composed of an Intel Realsense D435i camera and several ArUco markers.
The effectiveness of the proposed method was demonstrated through an obstacle avoidance task; and a task where the removal of obstacles is required. In both tasks, the planar slider was $\SI{8.0}{}\times\SI{15.0}{}\times\SI{5.0}{cm^3}$ in size, with estimated $\mu_p=0.1$ and measured frictional force $\SI{1.2}{N}$. The movable obstacle was a $\SI{7.0}{}\times\SI{12.2}{}\times\SI{5.0}{cm^3}$ cube. Finally, we tested the algorithm in a planar object retrieval task to validate the possibility of generalizing the method to obstacles of other geometric shapes, i.e., cylinders, and to validate the robustness when consecutively pushing away obstacles is needed. In all the experiments, the MPC prediction horizon was 30 steps, and we set $\Bar{f}=\SI{0.5}{N},\bar{\dot{\psi}}_c=\SI{3.0}{rad/s},\tau_{\text{MPC}}=\SI{0.04}{s}$, and set $\bar{\psi}_c=\SI{0.52}{},\SI{0.9}{rad}$ for the short and long edge of the slider, respectively.
\begin{figure}[t]
    \centering
    \includegraphics[width=1.0\columnwidth]{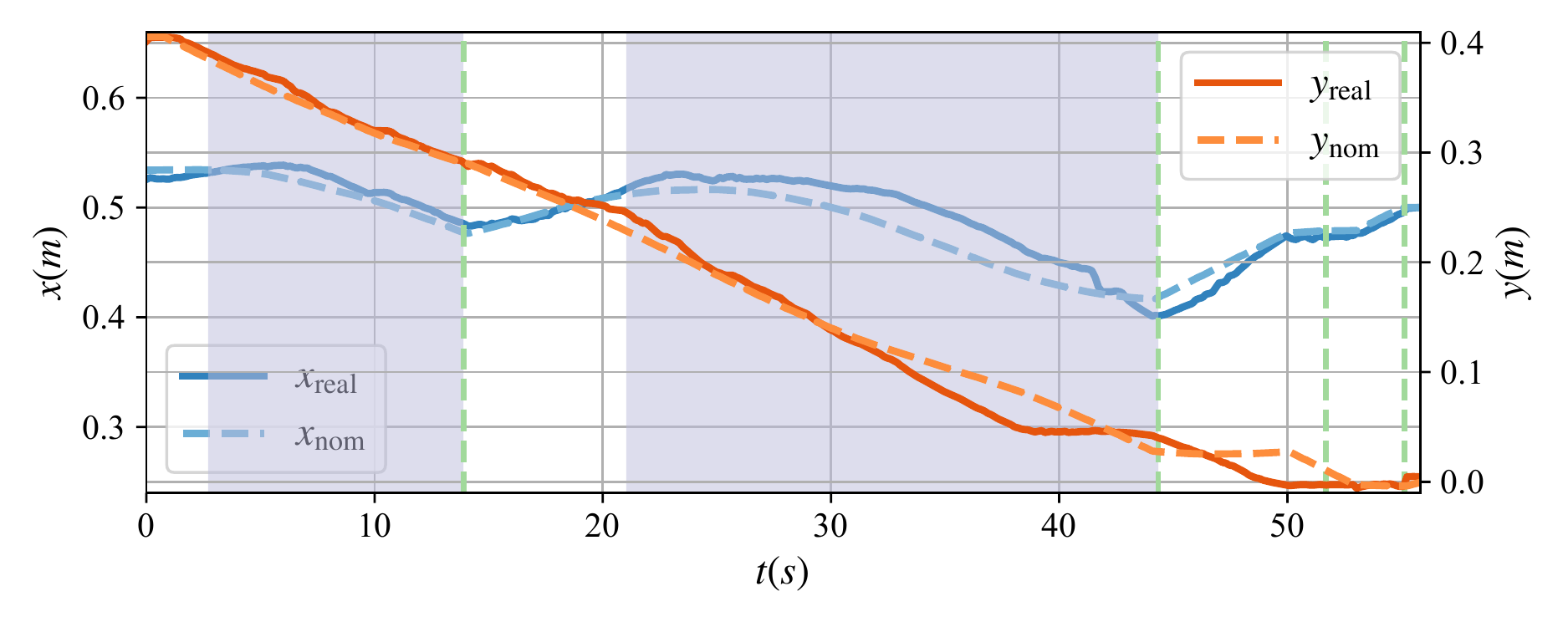}
    \vspace{-20pt}
    \caption{
        Tracking error of the object retrieval task. Green dashed lines mark the moments of switching faces; purple shadows report the intervals of contact. Blue and orange curves represent the x and y dimensions, respectively.
    }
    \label{fig: exp retrieval track error}
    \vspace{-20pt}
\end{figure}
The planned path, pusher, and slider trajectories in the obstacle avoidance task are presented in Fig. \ref{fig: exp avoidance trajectory}.
% including 6 times of switching face
 It is hard for the slider to move through the obstacles in its initial pose because the interspace is narrower than the side length. Hence, the slider steered to have a short edge ahead. Note that the task requires sharp turning at times, which relies on the face-switching technique to improve controllability. There were no redundant movements of the slider except for near the target, which can be further improved
 % We speculate that the roughly tuned distance metric and imprecise state connection lead to exploring towards the goal with lower efficiency. Such a shortfall could be overcome 
 by replacing the goal pose with the goal region.
As shown in Fig. \ref{fig: exp pushaway traj}, the slider is obstructed by one movable obstacle and two square objects fixed on either side. To make space and simultaneously approach the goal, the slider contacted, pushed, and detached from the obstacle in succession. Despite the inaccurate modeling and randomness of frictional contacts, there is not much difference in the final positions of the movable obstacle between execution and planning. Moreover, the disturbance rejection property of MPC allowed the system to recover from moderate tracking errors. As seen in Fig. \ref{fig: exp pushaway traj}, the pusher moved towards the edge of the contact face to increase the moment of the exerted force. Then, the slider successfully reached the target position with a small error; since the tracking error remained bounded in continuous contact and rapidly converged once the disturbance vanished.
Finally, Fig. \ref{fig: exp complete retrieval task} depicts five keyframes of a complete planar object retrieval task. The initial and goal positions of the planar slider are $\bm{x}^s[0]=\left[\SI{0.53}{m},\SI{0.41}{m},\SI{-1.57}{rad}\right]^{\top}$ and $\bm{x}^s[T]=\left[\SI{0.50}{m},\SI{0.00}{m},\SI{-3.14}{rad}\right]^{\top}$, respectively. The task scenerio contains fixed obstacles of $\SI{10.0}{}\times\SI{10.2}{}\times\SI{5.0}{cm^3}$ and $\SI{5.0}{}\times\SI{5.0}{}\times\SI{5.0}{cm^3}$, the additional cylindrical movable obstacle is of $\Phi\SI{7}{}\times\SI{6}{cm^3}$. Obstacles are simplified as their minimum bounding rectangles in CA3P, with an estimated frictional coefficient $\mu=\SI{0.3}{}$ between all pairs. Since it is challenging to control contact forces directly, we converted the control input $\bm{u}^p$ to speed command. The pusher is initialized at the center of each contact face the robot has switched to. Other parameters remain unchanged compared to the two previous experiments. As shown in Fig. \ref{fig: exp retrieval frame 0} and Fig. \ref{fig: exp retrieval frame 1}, the slider pushed the cylindrical object aside to enlarge the space ahead; instead of passively performing a time-consuming avoidance behavior. Later the slider passed through the narrow corridor and came into contact with another obstacle, as Fig. \ref{fig: exp retrieval frame 2} depicts. We observed the fast-moving behavior of the pusher on the slider's periphery as an anti-disturbance mechanism (Fig. \ref{fig: exp retrieval frame 3}). The slider eventually broke out of the clutter in $\SI{50.6}{s}$ (Fig. \ref{fig: exp retrieval frame 4}). The tracking error in the x and y directions are reported in Fig. \ref{fig: exp retrieval track error}.

\vspace{-5pt}
\section{CONCLUSIONS}
% \begin{figure}[!t]
%     \centering
%     \includegraphics[width=.9\columnwidth]{figs/ablation_mujoco.pdf}
%     \vspace{-10pt}
%     \caption{Learning curves comparing \textbf{dynamic TD3} against \textbf{fixed-step TD3} on the suite of \textbf{OpenAI gym continuous control tasks}~\cite{todorov2012mujoco}.
%     %
%     ``TD3 with fixed steps 2 (or 3)" means the agent repeats each action for 2 (or 3) times.
%     %
%     %  ``TD3 with fixed steps 2" and ``TD3 with fixed steps 3" represent the TD3 variants that repeats every action for 2 and 3 time steps respectively.
%      %
%     %  It turns out that when compared to fixed-step TD3, dynamic TD3 obtains equal or even better performance, with equal or even fewer inference times than any of them. \shenzhi{need to be larger}
%     }
%     \label{fig: mujoco ablation}
%     \vspace{-15pt}
% \end{figure}

% \vspace{-10pt}
This work proposes a new manipulation method for non-prehensile planar pushing in a constrained workspace. We combine sampling-based approaches with a simplified object interaction model for motion planning and apply the MPC scheme for robust control. 
With the use of those techniques together, the proposed method is with the novel contact-aware feature, 
%These techniques enable the contact-aware feature of the proposed framework, 
which allows the robot to actively avoid obstacles, switch contacts, or remove obstacles simultaneously. Multiple actions are integrated into the planning algorithm of CA3P, and its effectiveness has been comprehensively validated in the task of object retrieval, subject to several challenges (e.g., densely cluttered environments, uncertain physical parameters, and 
% underactuated end effector
complex kinodynamic constraints). 
% We have demonstrated that applying our method in a challenging object retrieval task is practical. 
% It is noteworthy that the physical parameters including frictional coefficients are not carefully measured. This work shows the potential to apply non-prehensile manipulation instead of traditional grasping for enhanced dexterity. Complex kinodynamic constrains and modeling uncertainties can be tackled with hierarchical planning and control.
%
%
% One limitation of the proposed framework is that optimality of the planned trajectory is not guaranteed, since the imprecise state connection of linearized dynamics will cause accumulation of errors. 
% %, which does no good to control. 
% Another limitation is that considering multiple object interactions and indirect contacts is impractical, and requires more effective control methods. 
Future works will be devoted to improving the quality of motion planning with trajectory optimization; and to taking account of higher-order dynamics for preferable dynamic non-prehensile manipulation.
%


%%%%%%%%%%%%%%%%%%%%%%%%%%%%%%%%%%%%%%%%%%%%%%%%%%%%%%%%%%%%%%%%%%%%%%%%%%%%%%%%

%%%%%%%%%%%%%%%%%%%%%%%%%%%%%%%%%%%%%%%%%%%%%%%%%%%%%%%%%%%%%%%%%%%%%%%%%%%%%%%%

%%%%%%%%%%%%%%%%%%%%%%%%%%%%%%%%%%%%%%%%%%%%%%%%%%%%%%%%%%%%%%%%%%%%%%%%%%%%%%%%
%\newpage

% \input{chapters/7_appendix}

% \section*{ACKNOWLEDGMENT}
% \input{chapters/8_acknowledgment}
% {\small
% \bibliographystyle{ref/IEEEtran}
% \bibliography{ref/IEEEabrv, ref/ref}
% }
{\small
\bibliographystyle{IEEEtran}
\bibliography{ref}
}

%%%%%%%%%%%%%%%%%%%%%%%%%%%%%%%%%%%%%%%%
%%%%%%%%%%% bbl file - ArXiv %%%%%%%%%%%
%%%%%%%%%%%%%%%%%%%%%%%%%%%%%%%%%%%%%%%%
% \bibliography{main}

\end{document}